\documentclass[final]{cvpr}

\usepackage{times}
\usepackage{epsfig}
\usepackage{graphicx}
\usepackage{amsmath}
\usepackage{amssymb}
\usepackage{enumitem}
\usepackage{booktabs}
\usepackage{balance}
\usepackage{algorithmic}
\usepackage{algorithm}
\setitemize{noitemsep,topsep=0pt,parsep=0pt,partopsep=0pt}
\newtheorem{myDef}{Definition} 
\newtheorem{theorem}{Theorem}

\newtheorem{proof}{Proof}[section]

\usepackage[pagebackref=false,breaklinks=true,letterpaper=true,colorlinks,urlcolor=blue,citecolor=blue,linkcolor=blue,bookmarks=false]{hyperref}

\def\cvprPaperID{1398} % *** Enter the CVPR Paper ID here
\def\confYear{CVPR 2021}

\begin{document}

%%%%%%%%% TITLE
\title{ArtFlow: Unbiased Image Style Transfer via Reversible Neural Flows}

\author{Jie An$^{1}$\thanks{J. An and S. Huang contribute equally. This work is done when J. An is an intern in Tencent AI Lab. The code is available at \url{https://github.com/pkuanjie/ArtFlow}.}\quad
Siyu Huang$^{2}$\footnotemark[1] \quad
Yibing Song$^3$ \quad
Dejing Dou$^2$ \quad
Wei Liu$^4$ \quad
Jiebo Luo$^1$ \\
$^1$University of Rochester \quad
$^2$Baidu Research \quad 
$^3$Tencent AI Lab \quad
$^4$Tencent Data Platform \\
{\tt\small \{jan6,jluo\}@cs.rochester.edu} \quad 
{\tt\small \{huangsiyu,doudejing\}@baidu.com}\\
{\tt\small yibingsong.cv@gmail.com} \quad
{\tt\small wl2223@columbia.edu}\\
}
\maketitle
\pagestyle{empty}  % no page number for the second and the later pages
\thispagestyle{empty} % no page number for the first page

%%%%%%%%% ABSTRACT
% \begin{abstract}
%     Neural style transfer is an image editing task aiming at transferring the artistic style from a reference image to the input content image. Starting from Gatys~\etal, significant advances have been made in style transfer. Recently proposed universal style transfer algorithms achieve both good stylization effects and great generalization ability. However, those algorithms are all biased towards preserving more style information. Consequently, the content information are leaked in style transfer. To address this issue, we propose a new style transfer algorithm named ArtFlow. The proposed ArtFlow is constructed based on neural flows in conjunction with AdaIN, and hereby supports both forward and reverse inferences. In style transfer, the forward inference is to project images into deep features and the reverse inference recovers features back to images in an unbiased way. Extensive experiments demonstrate that ArtFlow can achieve comparable style transfer performance against state-of-the-art methods while avoiding the content leak phenomenon. 
% \end{abstract}

% written by SIYU
\begin{abstract}
Universal style transfer retains styles from reference images in content images. While existing methods have achieved state-of-the-art style transfer performance, they are not aware of the content leak phenomenon that the image content may corrupt after several rounds of stylization process. In this paper, we propose ArtFlow to prevent content leak during universal style transfer. ArtFlow consists of reversible neural flows and an unbiased feature transfer module. It supports both forward and backward inferences and operates in a projection-transfer-reversion scheme. The forward inference projects input images into deep features, while the backward inference remaps deep features back to input images in a lossless and unbiased way. Extensive experiments demonstrate that ArtFlow achieves comparable performance to  state-of-the-art style transfer methods while avoiding content leak.
%Existing universal style transfer methods have achieved state-of-the-art stylization effects and generalization abilities. However, empirical evidence reveals that the content information can be severely leaked during the style transfer process by those algorithms. To address this issue, this paper presents a novel universal style transfer framework named ArtFlow. ArtFlow consists of reversible neural flows in conjunction with an unbiased feature transfer module. ArtFlow supports both forward and reverse inferences and works in a projection-transfer-reversion scheme. The forward inference projects images into deep features, and the reverse inference recovers features back to images in a lossless and unbiased way. Extensive experiments demonstrate that the proposed ArtFlow achieves comparable style transfer performance to the state-of-the-art methods while avoiding the content leak issue.
\end{abstract}

%%%%%%%%% BODY

\section{Introduction}
Neural style transfer aims at transferring the artistic style from a reference image to a content image. 
Starting from \cite{gatys2015texture,gatys2015neural}, numerous works based on iterative optimization~\cite{gatys2016preserving,risser2017stable,li2017laplacian,li2017demystifying} and feed-forward networks~\cite{johnson2016perceptual,ulyanov2016texture,chen2017stylebank,zhang2017multi} improve style transfer from either visual quality or computational efficiency.
Despite tremendous efforts, these methods do not generalize well for multiple types of style transfer. Universal style transfer (UST) is proposed to improve this generalization ability. 
%
%all these methods can only handle limited styles per training/optimization. 
%
%To improve the generalization ability of style transfer, universal style transfer (UST) methods~\cite{chen2016fast,li2017universal,huang2017arbitrary,sheng2018avatar,li2018learning,lu2019optimal} are proposed. 
%
The representative UST methods include AdaIN~\cite{huang2017arbitrary}, WCT~\cite{li2017universal}, and Avatar-Net~\cite{sheng2018avatar}. 
These methods are continuously extended by~\cite{gu2018arbitrary,jing2020dynamic,wang2020diversified,an2020real,sheng2018avatar,li2018closed,lu2019optimal,li2018learning,an2019ultrafast,wang2020collaborative}.
While achieving favorable results as well as generalizations, these methods are limited to disentangling and reconstructing image content during the stylization process.
%
% QUESTION: how to express the importance of the content leak?
%These benchmark methods achieve good style transfer results while having improved generalization abilities. However, in terms of the theory and mechanism behind the algorithms,  they all have an implicit but significant theoretical flaw -- The content information is not faithfully disentangled, reconstructed or preserved during style transfer. 
Fig.~\ref{fig:content_leak} shows some examples. 
Existing methods~\cite{li2017universal,huang2017arbitrary,sheng2018avatar} effectively stylize content images in (c).
However, image contents are corrupted after several rounds of stylization process where we send the reference image and the output result into these methods.
We define this phenomenon as content leak and provide an analysis in the following:
%In normal style transfer results (c) by those algorithms, the content objects, \eg, the bus, building, and street, are recognizable. However, when we perform style transfer on an image repeatedly for several rounds, as shown in (d), most of the content information disintegrates. This phenomenon indicates that a part of content information is lost during the style transfer process of those algorithms. In this paper, we call this phenomenon Content Leak.
%Regarding the style transfer performance, 
\begin{figure}[t]
\centering
\includegraphics[width=\linewidth]{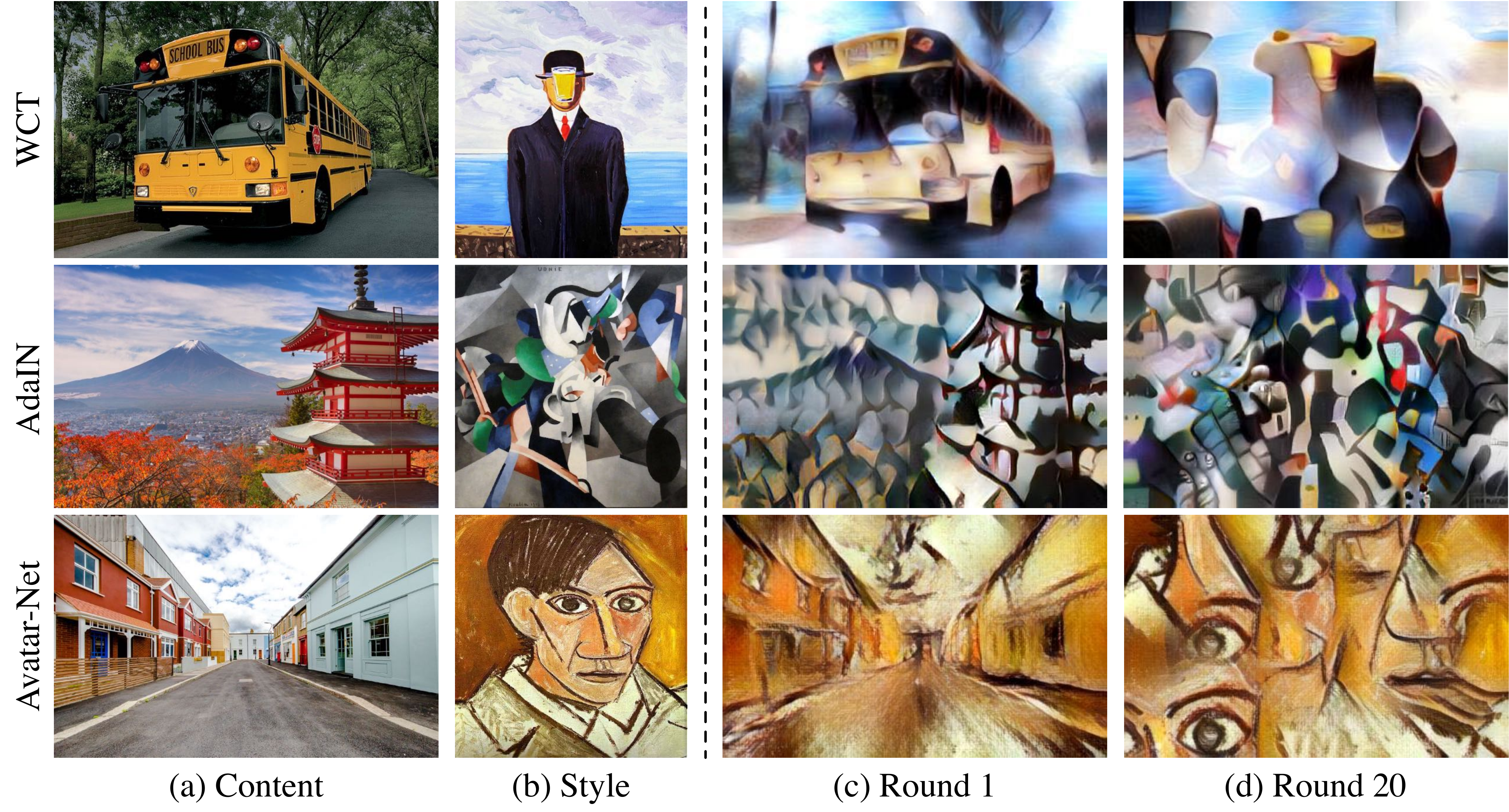}
\caption{Content leak visualization. Existing style transfer methods are not effective to preserve image content after several rounds of stylization process as shown in (d), although their performance is state-of-the-art in the first round as shown in (c).}
%Visualization of the Content Leak phenomenon. For each method, we first perform style transfer with the input content-style pair. Then we use the stylized image as the new content and perform style transfer for 20 rounds. We show  stylization results in the first (c) round and the 20-th round (d). with content leak.}
\label{fig:content_leak}
\end{figure}

Content leak appears due to the design of UST methods that usually consist of three parts: the first part is a fixed encoder for image embedding, the second part is a learnable decoder to remap deep features back to images, and the third part is a style transfer module based on deep features.
%
%Why does Content Leak happen? To uncover the causes of the Content Leak phenomenon, we first summarize the mechanism of the state-of-the-art UST methods. Existing UST methods usually consist of three parts: 1) a fixed \emph{encoder} to embed images into deep feature space, which is usually a VGG19~\cite{simonyan2014very} or GoogLeNet~\cite{szegedy2015going,an2020real} pre-trained on the ImageNet dataset~\cite{deng2009imagenet}, 2) a learnable \emph{decoder} to invert deep features back to the image space, and 3) a \emph{transfer module} to perform style transfer upon deep features. For example, the encoder and decoder structures of AdaIN, WCT, and Avatar-Net are very similar, but their transfer modules are distinct from each other.
%
We observe that the first part is fixed. The appearance of content leak indicates 
the accumulated image reconstruction errors brought by the decoder, or the biased training process of either the decoder or the style transfer module.
%
%Since the encoder of UST algorithms is fixed, if the decoder and the transfer module do not bring accumulated error or biases towards style, Content Leak would not happen. 
%
%Therefore, the Content Leak phenomenon is caused by the image reconstruction error of the decoder, biased decoder training, or biased style transfer modules. 
%
Specifically, the content leaks of WCT~\cite{li2017universal} and its variants~\cite{li2018learning,lu2019optimal,wang2020collaborative} is mainly caused by the image reconstruction error of the decoder. The content leak of AdaIN series~\cite{huang2017arbitrary,jing2020dynamic,wang2020diversified} and Avatar-Net~\cite{sheng2018avatar} are additionally caused by the biased decoder training and a biased style transfer module, respectively. Sec.~\ref{sec:analysis} shows more analyses.

In this work, we propose an unbiased style transfer framework called ArtFlow to 
robustify exisiting UST methods upon overcoming content leak.
%address the Content Leak issue of AdaIN, WCT, Avatar-Net and their variances.
Different from the prevalent encoder-transfer-decoder structure, ArtFlow introduces both forward and backward inferences to formulate a projection-transfer-reversion pipeline.
%
%Different from existing approaches adopting the \emph{encoder-transfer-decoder} scheme to perform style transfer, ArtFlow supports both forward and reverse inferences, and therefore forms a new style transfer scheme, \ie, \emph{projection-transfer-reversion}. 
This pipeline is based on neural flows~\cite{dinh2014nice} and only contains a Projection Flow Network (PFN) in conjunction with an unbiased feature transfer module. 
The neural flow refers to a number of deep generative models \cite{dinh2014nice,ho2019flow++} which estimate density through a series of reversible transformations. 
Our PFN follows the neural flow model GLOW \cite{kingma2018glow} which consists of a chain of revertible operators including activation normalization layers, invertible $1\times 1$ convolutions, and affine coupling layers \cite{dinh2016realnvp}. 
% Because of the reversibility of these flow operators, PFN supports both forward and reverse inferences, and therefore allows a new style transfer scheme, \ie, \emph{forward-transfer-reverse}, which is completely different from existing UST methods based on auto-encoders. 
Fig.~\ref{fig:intro_framework} shows the structure of ArtFlow. It first projects both the content and style images into latent representations via forward inference.
Then, it makes unbiased style transfer upon deep features and reconstructs the stylized images via reversed feature inference. 

\begin{figure}[t]
	\centering
	\includegraphics[width=\linewidth]{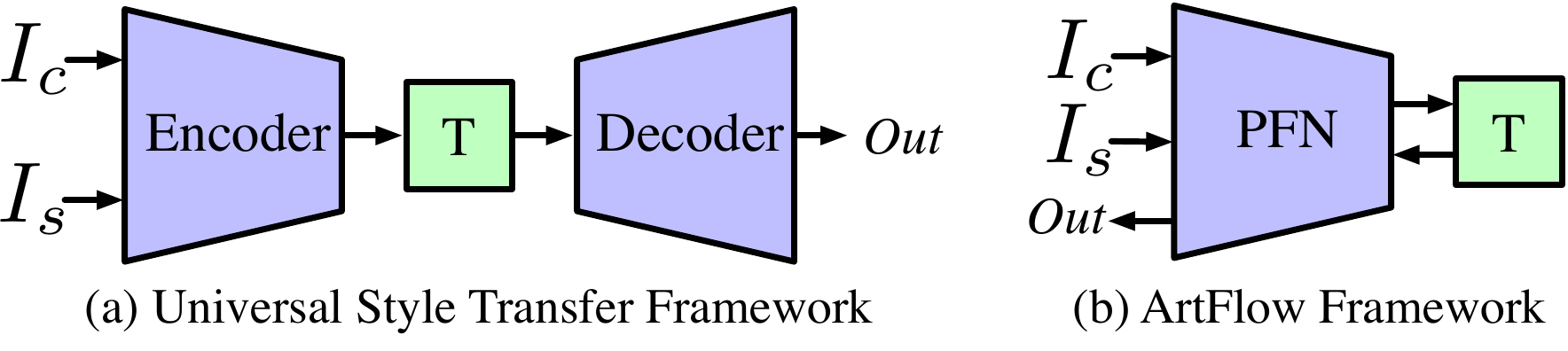}
	\caption{A comparison between the auto-encoder based style transfer framework and the proposed ArtFlow framework.}
	\label{fig:intro_framework}
\end{figure}

The proposed PFN avoids the image reconstruction error and image recovery bias which usually appear in the encoder-decoder framework.
PFN allows unbiased and lossless feature extraction and image recovery. 
To this end, PFN facilitates the comparison of style transfer modules in a fair manner. Based on PFN, we perform theoretical and empirical analyses of the inherent biases of style transfer modules adopted by WCT, AdaIN, and Avatar-Net. We show that the transfer modules of AdaIN and WCT are unbiased, while the transfer module of Avatar-Net is biased towards style. Consequently, we adopt the transfer modules of AdaIN and WCT as the transfer modules for ArtFlow to achieve an unbiased style transfer.

The contributions of this work are three-fold:
\begin{itemize}
    \item We reveal the Content Leak issue of the state-of-the-art style transfer algorithms and identify the three main causes of the Content Leak in AdaIN~\cite{huang2017arbitrary}, WCT~\cite{li2017universal}, and Avatar-Net~\cite{sheng2018avatar}.
    \item We propose an unbiased, lossless, and reversible network named PFN based on neural flows, which allows both theoretical and empirical analyses of the inherent biases of the popular style transfer modules.
    \item Based on PFN in conjunction with an unbiased style transfer module, we propose a novel style transfer framework, \ie, ArtFlow, which achieves comparable style transfer results to state-of-the-art methods while avoiding the Content Leak issue.
\end{itemize}

\section{Related Work}

{\flushleft \bf Image Style Transfer}.
Image style transfer is a long-standing research topic. Before deep neural networks~\cite{song2017crest,liu2020rethinking,wang2020rethinking,wang2021unsupervised} are applied to the style transfer, several algorithms based on stroke rendering~\cite{hertzmann1998painterly}, image analogy~\cite{hertzmann2001image,shih2014style,frigo2016split,song2017stylizing,liao2017visual,song2017learning,song2019joint}, and image filtering~\cite{winnemoller2006real} are proposed to make artistic style transfer. These methods usually have to trade-off between style transfer quality, generalization, and efficiency. Gatys~\etal \cite{gatys2015texture,gatys2015neural} introduce a Gram loss upon deep features to represent image styles, which opens up the neural style transfer era. Inspired by Gatys~\etal, numerous neural style transfer methods have been proposed. We categorize these methods into one style per model~\cite{li2016combining, ulyanov1607instance, johnson2016perceptual, ulyanov2016texture, ulyanov2017improved, wang2017multimodal, risser2017stable, li2017laplacian}, multi-style per model~\cite{dumoulin2017learned, chen2017stylebank, huang2017real, gong2018neural}, and universal style transfer methods~\cite{chen2016fast,li2017universal,huang2017arbitrary,sheng2018avatar,gu2018arbitrary,an2019ultrafast,wang2020collaborative,liu2020geometric,wang2020diversified,an2020real} with respect to their generalization abilities. 
In this paper, our ArtFlow belongs to universal style transfer and it consists of reversible neural flows. The forward and backward inferences are utilized for lossless and unbiased image recovery.

%The method proposed in this paper belongs to the universal style transfer category. Most universal style transfer methods have a similar framework, which adopts an encoder, a decoder, and a transfer module. The main difference between the proposed ArtFlow and other universal style transfer methods is that we build the backbone network based on the reversible neural flows. Instead of training a separate decoder, ArtFlow uses the forward and backward inference of the flow-based backbone to replace the encoder and the decoder, respectively. Since the neural flows are capable of a lossless and unbiased image recovery, ArtFlow achieves an unbiased style transfer in conjunction with using AdaIN/WCT as the feature transfer module.

{\flushleft \bf Neural flows}.
Neural flows refer to a subclass of deep generative models, which learns the exact likelihood of high dimensional observations (\eg, natural images, texts, and audios) through a chain of reversible transformations. As a pioneering work of neural flows, NICE \cite{dinh2014nice} is proposed to transform low dimensional densities to high dimensional observations with a stack of affine coupling layers. Following NICE, a series of neural flows, including RealNVP \cite{dinh2016realnvp}, GLOW \cite{kingma2018glow}, and Flow++ \cite{ho2019flow++}, are proposed to improve NICE with more powerful and flexible reversible transformations. The recently proposed neural flows \cite{kingma2018glow,ho2019flow++,ma2019macow} are capable of synthesizing high-resolution natural/face images, realistic speech data \cite{prenger2019waveglow,kim2020glowtts}, and performing make-up transfer~\cite{chen2019beautyglow}. In this work, the proposed ArtFlow consists of a reversible network PFN and an unbiased feature transfer module. The content leak can be addressed via lossless forward and backward inferences and unbiased feature transfer. In comparison, BeautyGlow~\cite{chen2019beautyglow} shares the similar spirits but is not applicable for unbiased style transfer.

\begin{figure}
	\centering
	\includegraphics[width=\linewidth]{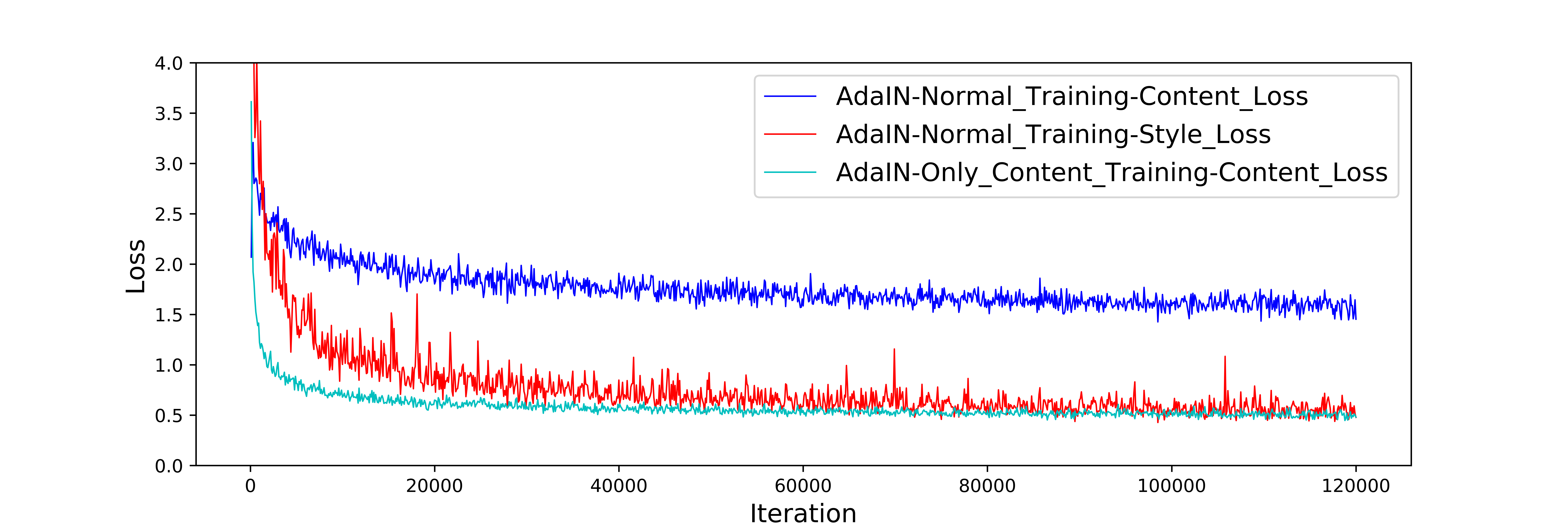}
	\caption{Loss curves of AdaIN \cite{huang2017arbitrary} training: Using both content and style losses \vs only using the content loss.}
	\label{fig:adain_loss}
\end{figure}

\section{Pre-analysis}
\label{sec:analysis}
Before introducing the proposed ArtFlow, we first make a pre-analysis to uncover the Content Leak phenomenon of the state-of-the-art style transfer algorithms and analyze the causes of Content Leak. We make the aforementioned pre-analysis by answering two questions: What Content Leak is and why Content Leak happens.

\subsection{What is Content Leak?}
For a style transfer algorithm, Content Leak occurs because the stylization results lose some content information. Although the existing state-of-the-art style transfer algorithms, 
\eg, AdaIN~\cite{huang2017arbitrary}, WCT~\cite{li2017universal}, and Avatar-Net~\cite{sheng2018avatar}, can produce good style transfer results, they still suffer from the Content Leak issue. Since it is hard to directly extract the content information from the stylized image and compare it with the input content image, we adopt an alternative way to show empirical evidence of the Content Leak phenomenon. More specifically, we first perform the style transfer with an input content-style pair based on a style transfer algorithm. We then take the stylized image as the new content and repeatedly perform the style transfer process for 20 times. Fig.~\ref{fig:content_leak} shows the results of our experiments for AdaIN (row 1), WCT (row 2), and Avatar-Net (row 3). 
% For each algorithm, we conduct two separate experiments, \ie, a single style transfer and a multi style transfer, to avoid the potential misleading caused by specific style images and improve the robustness of our empirical studies. In first experiment (row 1, 3), we use a fixed style image as shown in Fig.~\ref{fig:content_leak} (a). Conversely, in the other experiment  (row 2, 4), we use a sequence of style images and change the style image per round. 
According to Fig.~\ref{fig:content_leak}, when we perform style transfer for 20 rounds, we can hardly recognize any detail of the content image. 
Such an empirical evidence 
% holds for AdaIN and WCT in both single style and multi style settings, which 
indicates that the Content Leak phenomenon occurs in all AdaIN, WCT, and Avatar-Net. In the following, we discuss the causes of the Content Leak, which imply that the Content Leak issue also exists in other state-of-the-art style transfer algorithms.
%Regarding the analysis to the Content Leak issue of other state-of-the-art style transfer algorithms, we leave them for simplicity. Instead, 

\begin{figure}
    \centering
    \includegraphics[width=\linewidth]{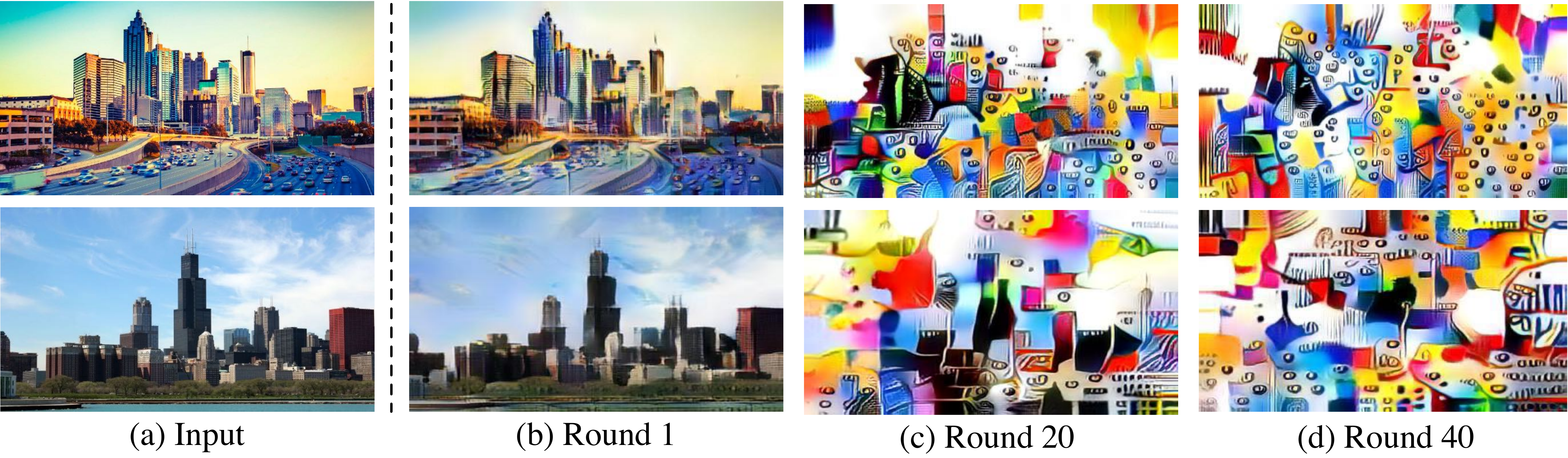}
    \caption{Multiple rounds of image encoding and decoding using the auto-encoder of AdaIN~\cite{huang2017arbitrary}. }
    \label{fig:content_reconstruction}
    %In the first round, we feed an image into the auto-encoder. Then we repeatedly feed the outputs into the auto-encoder for 40 times.
\end{figure}

\begin{figure*}[t]
\centering
\includegraphics[width=\linewidth]{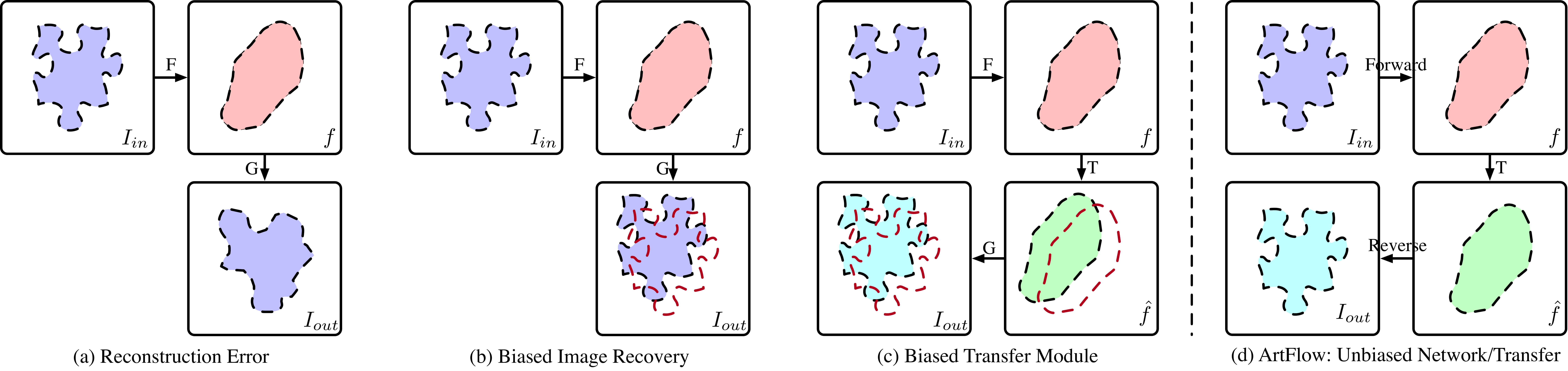}
\caption{Causes of the Content Leak phenomenon. (a) Reconstruction error, \ie, the content of output image is disturbed. (b) Biased image recovery, \ie, the output image shifts to a biased style via the decoder. (c) Biased style transfer module, \ie, the stylized feature shifts to a biased style via feature stylization. The red dash lines in (b) and (c) denote unbiased positions of the manifolds. (d) The proposed ArtFlow scheme, where both the network and the transfer module are not biased, while the backbone network does not introduce any reconstruction error. Notations -- $F$ and $G$: encoder and decoder used by existing style transfer algorithms, respectively. $T$: style transfer module, \eg, AdaIN or WCT. $I_{in}$ and $I_{out}$: input and output images. $f$ and $\hat{f}$: vanilla and stylized deep features of $I_{in}$. }
\label{fig:bias_vs_unbias}
\end{figure*}

\subsection{Why Does Content Leak Happen?}
\label{causes}
Taking AdaIN~\cite{gu2018arbitrary}, WCT~\cite{li2017universal}, and Avatar-Net~\cite{sheng2018avatar} as three representatives of style transfer algorithms, we study the causes of the Content Leak phenomenon.

%\vspace{-1mm}
\noindent\textbf{Reconstruction error.} A straightforward explanation to Content Leak is that the decoder of existing style transfer algorithms cannot achieve lossless image reconstruction of the input content image. For example, all AdaIN, WCT, and Avatar-Net adopt VGG19~\cite{simonyan2014very} as the encoder and train a structurally symmetrical decoder to invert the features of VGG19 back to the image space. Although an image reconstruction loss~\cite{li2017universal} or a content loss~\cite{huang2017arbitrary} is used to train the decoder, Li \etal \cite{li2017universal} acknowledge that the decoder is far from perfect due to the loss of spatial information brought by the pooling operations in the encoder. Consequently, the accumulated image reconstruction error may gradually disturb the content details and lead to the Content Leak.

%\vspace{-1mm}
\noindent\textbf{Biased decoder training.} The above-mentioned reconstruction error can only partially explain the Content Leak phenomenon. In addition, biased decoder training is another cause. We take the training scheme of AdaIN as an example to explain how its loss function settings lead to Content Leak. AdaIN trains the decoder with a weighted combination of a content loss $L_c$ and a style loss $L_s$, where
\begin{align}
    L_c &= \|F(G(t)) - t\|_2, \label{eq:content_loss} \\
    L_s &= \sum\limits_{i=1...L} \| \mu(\phi_i(G(t))) - \mu(\phi_i(s)) \|_2 \label{eq:style_loss} \\ 
    &+ \sum\limits_{i=1...L} \| \sigma(\phi_i(G(t))) - \sigma(\phi_i(s)) \|_2. \notag 
\end{align}
Here $t$ denotes the output of the adaptive instance normalization, $F$ and $G$ represent the encoder and the decoder, respectively, $\phi_i$ denotes a layer in VGG19 used to compute the style loss, and $\mu,\sigma$ represent the mean and standard deviation of feature maps, respectively. Due to $L_s$, the decoder is trained to trade off between $L_c$ and $L_s$, rather than trying to reconstruct images perfectly. Fig.~\ref{fig:adain_loss} shows the training loss curves of AdaIN with and without $L_s$. When we train the decoder of AdaIN with only $L_c$, the converged value of $L_c$ (cyan curve) is significantly smaller than training with the weighted combination of $L_c$ and $L_s$ (blue curve). 
Consequently, the auto-encoder of AdaIN is biased towards rendering more artistic effects, which causes Content Leak. Fig.~\ref{fig:content_reconstruction} shows the image reconstruction results by propagating through the auto-encoder of AdaIN for 50 rounds. We take the output of the auto-encoder in the previous round as the input of the next round and perform image reconstruction repeatedly. With the increase of the inference rounds, weird artistic patterns gradually appear in the produced results, which indicates that the auto-encoder of AdaIN may memorize image styles in training and bias towards the training styles in inference. 
% To further demonstrate the biased decoder training in AdaIN, we compare the converged loss curves between the training scheme with both the content and style loss (normal training scheme) and only using the content loss. As Fig.~\ref{fig:adain_loss} shows, in the normal training scheme, the converged content loss is significantly higher than only using the content loss for training, which indicates that the style loss makes the decoder biased towards rendering more artistic effects.

\begin{figure*}
	\centering
	\includegraphics[width=\textwidth]{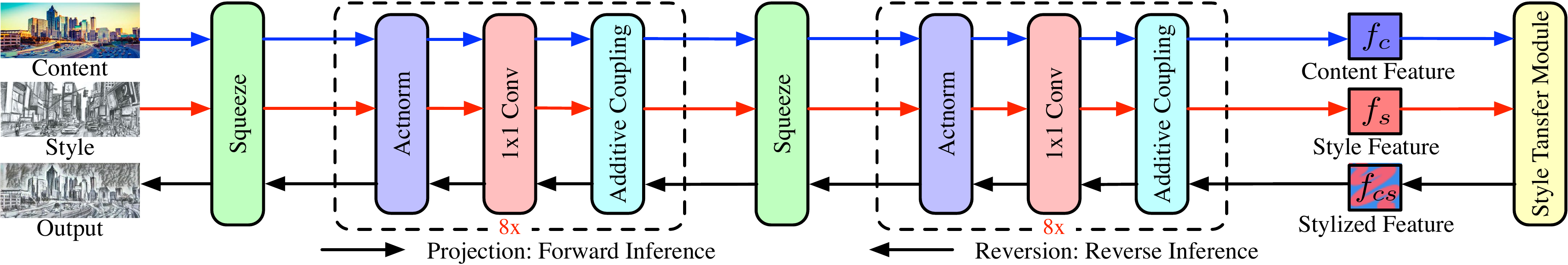}
	\caption{The framework of the proposed ArtFlow. For style transfer, Artflow works in a \emph{projection-transfer-reversion} scheme. \emph{Projection}: Extracting deep features of content and style images via the forward inference. \emph{Transfer}: Transferring the content and style features to the stylized feature via the style transfer module. \emph{Reversion}: Transforming the stylized feature to the stylized image via the reverse inference.}
	\label{fig:framework}
\end{figure*}

%\vspace{-2mm}
\noindent\textbf{Biased style transfer module.} Biased style transfer module is another cause of the Content Leak. We take the Style Decorator in Avatar-Net as an example. For the normalized content feature $f_c$ and style feature $f_s$, the key mechanism of the Style Decorator is motivated by the deep image analogy~\cite{liao2017visual}, which is composed of two steps. In the first step, the algorithm finds a corresponding patch in $f_s$ for every patch in $f_c$ according to the content similarity between two patches. In the next step, $f_{cs}$ is formed by replacing patches in $f_c$ with the corresponding patches in $f_s$. Since such a patch replacement is irreversible, $f_c$ cannot be recovered from $f_{cs}$, which makes $f_{cs}$ be biased towards style and consequently causes the Content Leak phenomenon.

We summarize and illustrate three main causes of Content Leak in Fig.~\ref{fig:bias_vs_unbias}. While the reconstruction error may disturb the content information in the output image, the biased image recovery and the biased transfer module may lead to a style shift in the output image. %Our method avoids all these issues and can achieve unbiased style transfer.

% We take the whitening and coloring transformation in WCT~\cite{li2017universal} for example. For simplicity, we neglect the centering and re-centering process of the WCT and concentrate on the whitening and coloring transforms,
% \begin{align}
%     &\mathrm{Whitening:} \hat{f}_c = E_c D_c^{-1/2} E_c^T f_c,\\
%     &\mathrm{Coloring:} f_{cs} = E_s D_s^{1/2} E_s^T \hat{f}_c.
% \end{align}
% where $f_c, f_{cs}$ represent the content and style-transferred features, $f_c f_c^T = E_c D_c E_c^T$ and $f_s f_s^T = E_s D_s E_s^T$. Therefore, the style-transferred features $f_{cs}$ can be expressed as,
% \begin{equation}
%     f_{cs} = E_s D_s^{1/2} E_s^T E_c D_c^{-1/2} E_c^T f_c.
% \end{equation}
% If the transfer module of WCT is unbiased, the output of the whitening transform to $f_{cs}$ should be equal to $\hat{f}_c$. However, since
% \begin{equation}
%     f_{cs}f_{cs}^T = D_s D_c^{-1} f_c f_c^T = E_c D_s E_c^T,
% \end{equation}
% we have,
% \begin{equation}
%     \mathrm{Whitening}(f_{cs}) = E_c D_s^{-1/2} E_c^T f_c \ne \hat{f}_c.
%     \label{eq:biased_transfer}
% \end{equation}
% Eq.~\ref{eq:biased_transfer} demonstrates that the whitening and coloring transforms in WCT is biased towards style and explains the cause of the Content Leak phenomenon in stylization results of WCT.

\section{Method}
\subsection{Overview of the ArtFlow Framework}
In this work, we present a novel unbiased style transfer framework named ArtFlow to address the Content Leak issue of the state-of-the-art style transfer approaches. Different from the \emph{encoder-transfer-decoder} scheme commonly used in existing neural style transfer algorithms, ArtFlow performs image style transfer through a \emph{projection-transfer-reversion} scheme. As shown in Fig.~\ref{fig:framework}, ArtFlow relies on a reversible neural flow model, named Projection Flow Network (PFN). In the \emph{projection} step, the content images and style images are fed into PFN for lossless deep feature extraction via the forward propagation of PFN. In the \emph{transfer} step, the content and style features are transferred to the stylized feature with an unbiased style transfer module. In the \emph{reversion} step, the stylized feature is reconstructed to a stylized image via the reverse propagation of PFN. Since the information flow in PFN and the unbiased style transfer module are both lossless and unbiased, ArtFlow achieves unbiased image style transfer to avoid the Content Leak.

In the following, we first discuss the details of PFN in Section \ref{sec:PFN}. Then, we discuss the choice of the unbiased style transfer module by performing both theoretical and quantitative analyses of the inherent biases of existing transfer modules in Section \ref{sec:separation}.

\subsection{Projection Flow Network}
\label{sec:PFN}
Projection Flow Network (PFN) serves as both the deep feature extractor and image synthesizer of our ArtFlow framework. In this work, we construct PFN by following the effective Glow model \cite{kingma2018glow}. As shown in Fig. \ref{fig:framework}, PFN consists of a chain of three learnable reversible transformations, \ie, additive coupling, invertible 1$\times$1 convolution, and Actnorm. All the components of PFN are reversible, making PFN fully reversible that the information is lossless during the forward and reverse propagation. In the following, we describe the three reversible transformations.

\noindent\textbf{Additive coupling.} 
Dinh \etal \cite{dinh2014nice,dinh2016realnvp} proposed an expressive reversible transformation named affine coupling layer. In this work, we adopt a special case of affine coupling, \ie, additive coupling, for PFN. The forward computation of additive coupling is
\begin{displaymath}
\begin{array}{rcl}
x_a, x_b & = & \mathrm{split}(x) \\
y_b & = & \mathrm{NN}(x_a) + x_b \\
y & = & \mathrm{concat}(x_a, y_b).
\end{array}
\end{displaymath}
The $\mathrm{split}()$ function splits a tensor into two halves along the channel dimension. $\mathrm{NN}()$ is (any) neural network where the input and the output have the same shape. The $\mathrm{concat}()$ function concatenates two tensors along the channel dimension. The reverse computation of additive coupling can be easily derived.

% The reverse computation of additive coupling can be easily derived as
% \begin{displaymath}
% \begin{array}{rcl}
% y_a, y_b & = & \mathrm{split}(y) \\
% x_a & = & y_a \\
% x_b & = & y_b - \mathrm{NN}(y_a)\\
% x & = & \mathrm{concat}(x_a, x_b).
% \end{array}
% \end{displaymath}

We observe that a flow model with additive coupling layers is sufficient to handle the style transfer task in experiments. Moreover, the additive coupling is more efficient and stable than the affine coupling in model training. Therefore, we employ additive coupling instead of affine coupling as the expressive transformation layer in PFN.

\noindent\textbf{Invertible 1$\times$1 convolution.} 
Since the additive coupling layer only processes a half of the feature maps, it is necessary to permute the channel dimensions of feature maps, so that each dimension can affect all the other dimensions \cite{dinh2014nice,dinh2016realnvp}. 
%Channel permutation of feature maps is significant for information flow in neural flows 
We follow Glow \cite{kingma2018glow} to use a learnable invertible 1$\times$1 convolution layer for flexible channel permutation, as
\begin{equation}
y_{i, j} = W x_{i,j}.
\end{equation}
$W$ is the weight matrix of shape $c \times c$, where $c$ is the channel dimension of tensor $x$ and $y$. Its reverse function is $x_{i, j} = W^{-1} y_{i,j}$.
%Invertible 1$\times$1 convolution is an effective channel permutation operation as an alternative to the fixed channel permutation. 
% Its reverse function is
% \begin{equation}
% x_{i, j} = W^{-1} y_{i,j}.
% \end{equation}

\noindent\textbf{Actnorm.} 
We follow Glow \cite{kingma2018glow} to use the activation normalization layer (Actnorm) as an alternative to batch normalization \cite{ioffe2015batchnorm}. Actnorm performs per-channel affine transformation on tensor $x$, as
\begin{equation}
y_{i,j} = w \odot x_{i,j} + b,
\end{equation}
where $i,j$ denote a spatial position on the tensor. $w$ and $b$ are the scale and bias parameters of affine transformation, and they are learnable in model training. The reverse funciton is $x_{i,j} = (y_{i,j} -b) / w$.
%Applying batch normalization \cite{ioffe2015batchnorm} to deep neural networks enables a more stable and faster model training. 
%, to alleviate the training issues of large models. 
% The reverse funciton is 
% \begin{equation}
% x_{i,j} = (y_{i,j} -b) / w.
% \end{equation}

In addition to the three reversible transformations, the $\mathrm{squeeze}$ operation is inserted into certain parts of PFN to reduce the spatial size of 2D feature maps. The $\mathrm{squeeze}$ operation splits the features into smaller patches along the spatial dimension and then concatenates the patches along the channel dimension.

\subsection{Unbiased Content-Style Separation}
\label{sec:separation}
Which style transfer module should ArtFlow use to achieve the unbiased style transfer? To answer this question, we first make a theoretical analysis of the biases of two popular style transfer modules, \ie, the adaptive instance normalization in AdaIN, and the whitening and coloring transforms in WCT.

The mechanism of the universal style transfer methods can be regarded as a natural evolution of the bilinear model proposed by Tenenbaum and Freeman in~\cite{Tenenbaum2000}, which separates an image into a content factor $C$ and a style factor $S$ and then makes style transfer by replacing the style factor $S$ in the content image with that in the target image. Similarly, the universal style transfer methods assume that the content information and the style information in the deep feature space are disentangled explicitly~\cite{huang2017arbitrary,li2017universal,li2018learning,lu2019optimal,an2019ultrafast,an2020real,wang2020collaborative,jing2020dynamic,wang2020diversified} or implicitly~\cite{chen2016fast,sheng2018avatar}. For example, AdaIN~\cite{huang2017arbitrary} separates deep features into normalized feature maps and mean/std vectors, which can be regarded as the content factor $C$ and style factor $S$, respectively. 

Following the theoretical framework of the Bilinear Model~\cite{Tenenbaum2000}, we can define the unbiased style transfer as:
\begin{myDef}
Suppose we have a bilinear style transfer module $f_{cs} = C(f_c) S(f_s)$, where $C, S$ denote the content factor and the style factor in the bilinear model, respectively. $f_{cs}$ is an unbiased style transfer module if $C(f_{cs}) = C(f_c)$ and $S(f_{cs}) = S(f_s)$.
\label{def:unbiased}
\end{myDef}
Based on Def.~\ref{def:unbiased}, we have the following two theorems. 
%\vspace{-2mm}
\begin{theorem}
\label{theorem1}
The adaptive instance normalization in AdaIN is an unbiased style transfer module.
\end{theorem}
% \begin{proof}
% Without loss of generality, we assume both $f_c$ and $f_s$ is centered. Therefore, we have,
% \begin{equation}
%     f_{cs} = \frac{f_c}{\sigma(f_c)} \sigma(f_s),
% \end{equation}
% where,
% \begin{equation}
%     C(f) = \frac{f}{\sigma(f)},\ S(f) = \sigma(f).
% \end{equation}
% Since,
% \begin{equation}
%     \sigma(f_{cs}) = \sigma(\frac{f_c}{\sigma(f_c)}) \sigma(f_s) = 1 \cdot \sigma(f_s) = \sigma(f_s),
% \end{equation}
% we have,
% \begin{align}
%     C(f_{cs}) &= \frac{f_cs}{\sigma(f_cs)} = \frac{f_c}{\sigma(f_c)} = C(f_c), \\
%     S(f_{cs}) &= \sigma(f_{cs}) = \sigma(f_s).
% \end{align}
% Therefore, the adaptive instance normalization in AdaIN is unbiased. $\hfill\blacksquare$ 
% \end{proof}

% Similarly, we can proof that the whitening and coloring transforms in WCT are also unbiased,
%\vspace{-4mm}
\begin{theorem}
\label{theorem2}
The whitening and coloring transform in WCT is an unbiased style transfer module.
\end{theorem}
%\vspace{-2mm}
% \begin{proof}
% Without loss of generality, we assume both $f_c$ and $f_s$ is centered. Therefore, we have,
% \begin{align}
%     &\mathrm{Whitening:} \hat{f}_c = E_c D_c^{-1/2} E_c^T f_c,\\
%     &\mathrm{Coloring:} f_{cs} = E_s D_s^{1/2} E_s^T \hat{f}_c.
% \end{align}
% where $f_c, f_{cs}$ represent the content and style-transferred features, $f_c f_c^T = E_c D_c E_c^T$ and $f_s f_s^T = E_s D_s E_s^T$. Therefore, the style-transferred features $f_{cs}$ can be expressed as,
% \begin{equation}
%     f_{cs} = E_s D_s^{1/2} E_s^T E_c D_c^{-1/2} E_c^T f_c.
% \end{equation}
% Here
% \begin{equation}
%     C(f) = E D^{-1/2} E^T f, \ S(f) = E D^{1/2} E^T
% \end{equation}
% Since,
% \begin{equation}
%     f_{cs}f_{cs}^T = f_s f_s^T = E_s D_s E_s^T,
% \end{equation}
% we have,
% \begin{align}
%     C(f_{cs}) &= E_s D_s^{-1/2}E_s^T f_{cs} = E_c D_c^{-1/2} E_c^T f_c = C(f_c), \\
%     S(f_{cs}) &= E_s D_s^{1/2}E_s^T = S(f_s).
%     \label{eq:biased_transfer}
% \end{align}
% Therefore, the whitening and coloring transforms in WCT is unbiased. $\hfill\blacksquare$ 
% \end{proof}
The proofs for Theorems \ref{theorem1} and \ref{theorem2} can be found in the supplementary material. The Style Decorator in Avatar-Net \cite{sheng2018avatar} does not fit the bilinear model, while the empirical analysis in Sec.~\ref{causes} shows that \emph{Style Decorator is a biased style transfer module}.

In addition to the theoretical analyses, we also quantitatively verify the unbiased property of the transfer modules in AdaIN and WCT. 
% Similar to Def.~\ref{def:unbiased}, we have,
% \begin{myDef}
% Suppose we have a bilinear style transfer module $f_{cs} = C(f_c) S(f_s)$, where $C, S$ denotes the content factor and the style factor in the bilinear model. $f_{cs}$ is an unbiased style transfer module if $C(F(G(f_{cs}))) = C(f_c)$ and $S(F(G(f_{cs}))) = S(f_s)$, where $F, G$ are encoder and the decoder and $F(G(f)) = f$.
% \label{def:unbiased}
% \end{myDef}
Quantitatively studying the property of popular style transfer modules is an unsolved question because the auto-encoder used by existing universal style transfer methods has significant image reconstruction errors and may be biased towards styles as discussed in Sec.~\ref{causes}. Consequently, the produced style transfer results using auto-encoders cannot precisely reflect the effects of the style transfer modules upon deep features. 
% Therefore, $f(g(f_{cs})) \ne f_{cs}$, where $f$ and $g$ are the encoder and decoder, respectively. Consequently, the \emph{encoder-transfer-decoder} framework cannot be used to verify the bias/unbias property of the transfer module. 
% Besides the above-mentioned mathematical proof, since $\mathrm{PFN}_f(\mathrm{PFN}_r(f_{cs})) = f_{cs}$, we can experimentally verify the bias/unbias property of style transfer modules, where $\mathrm{PFN}_f$ and $\mathrm{PFN}_r$ denotes the forward and the reverse inference of the proposed PFN, respectively.
The proposed PFN addresses this issue. Specifically, if we take the forward inference and the reverse inference of the proposed PFN as the encoder and decoder, respectively, we can obtain a lossless and unbiased ``auto-encoder'' for style transfer, which can avoid the influence of the image reconstruction error and the biased image recovery 
brought by the decoder.
% and hereby allows quantitative verification of the properties (\eg, bias/unbias) of the popular style transfer modules upon deep features. 

\begin{figure}
    \centering
    \includegraphics[width=\linewidth]{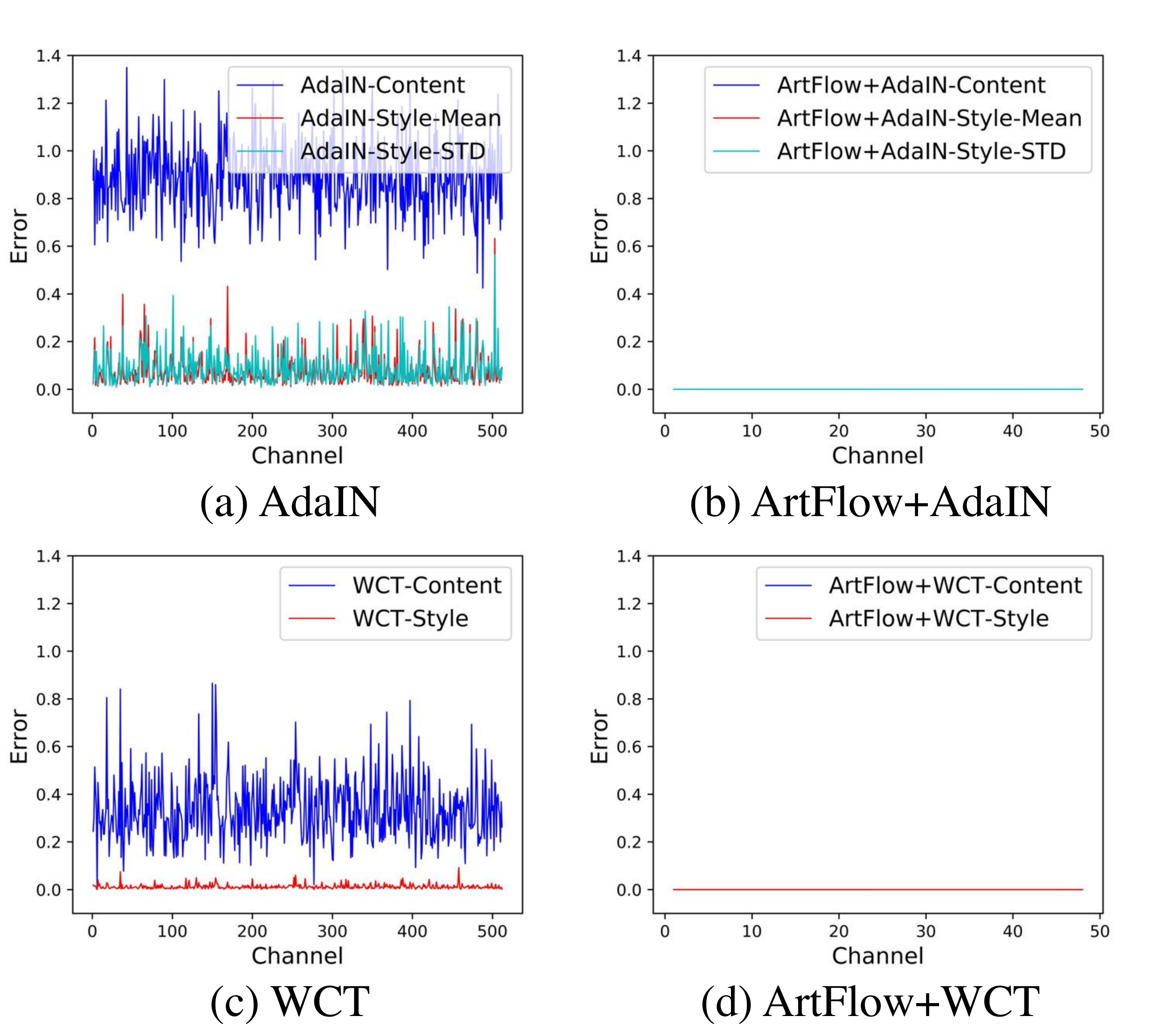}
    \caption{Content error between $F(G(f_{cs}))$ and $f_c$ and the style error between $F(G(f_{cs}))$ and $f_s$. ArtFlow with the transfer modules of AdaIN/WCT achieves lossless content/style reconstruction.}
%    \vspace{-4mm}
    \label{fig:error_curve}
\end{figure}
By using the proposed PFN as the lossless feature projector/inverter, we make a quantitative analysis about the content and style reconstruction errors of the transfer modules in AdaIN and WCT. Fig.~\ref{fig:error_curve} demonstrates two findings: 1) Considering (a) \vs (b) and (c) \vs (d), the proposed PFN can indeed make lossless and unbiased content and style reconstruction while the auto-encoder based on VGG19 cannot. 2) (b) and (d) quantitatively verify that the transfer module of AdaIN and WCT are unbiased. 

Based on theoretical and quantitative analyses to transfer modules in AdaIN and WCT, we let the adaptive instance normalization and the whitening and coloring transforms be two options for ArtFlow to achieve unbiased style transfer.

\begin{figure*}
	\centering
	\includegraphics[width=\textwidth]{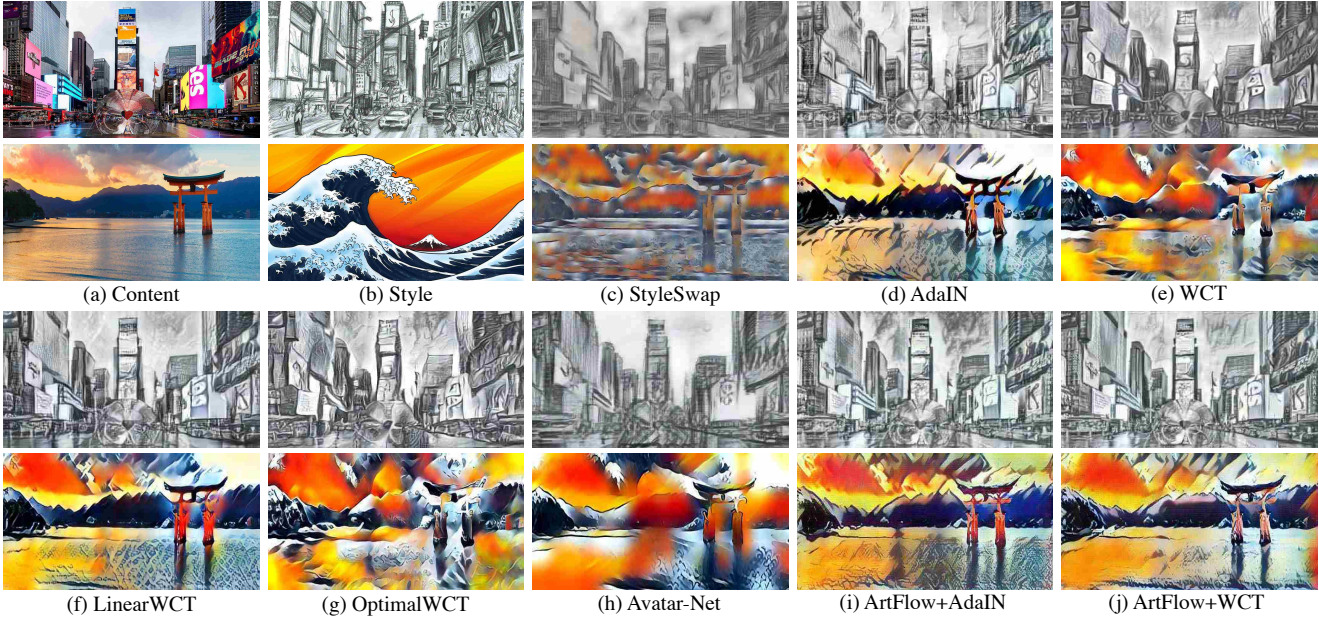}
	\caption{Style transfer results of the state-of-the-art universal style transfer algorithms.}
%	\vspace{-2mm}
	\label{fig:result}
\end{figure*}
% \begin{figure*}
% 	\centering
% 	\includegraphics[width=\textwidth]{./figures/losses}
% 	\caption{A comparison of the training losses of AdaIN and the proposed ArtFlow.}
% 	\vspace{-2mm}
% 	\label{fig:losses}
% \end{figure*}

\begin{table*}[t]
%%
%    \vspace{-1em}
    \caption{Quantitative evaluation results of universal stylization methods. The computing time is evaluated on 512$\times$512 images.}% using a RTX 2080Ti GPU.}
% 	\vspace{4mm}
    
\centering
    %\scriptsize 
    % \footnotesize
    %  \tabcolsep=0.15cm
    %  \renewcommand{\arraystretch}{1.2}
     \resizebox{1\hsize}{!}{
    \begin{tabular}{lcccccccccc}
        \hline
        Method & StyleSwap & AdaIN & WCT & LinearWCT & OptimalWCT & Avatar-Net & Self-Contained & ArtFlow+AdaIN & ArtFlow+WCT \\
        \hline
        SSIM~$\uparrow$ & 0.44 & 0.29 & 0.27 & 0.35 & 0.21 & 0.31 & 0.23 & 0.45 & \textbf{0.47}\\
        Content Loss~$\downarrow$ & \textbf{2.22} & 3.10 & 3.35 & 2.57 & 4.33 & 3.35 & 3.00 & 2.58 & 2.80\\
        Gram Loss~$\downarrow$ & 0.00482 & 0.00127 & 0.00074 & 0.00093 & \textbf{0.00035} & 0.00099 & 0.00473 & 0.00098 & 0.00078 \\
        Time (seconds)~$\downarrow$ & 0.272 & \textbf{0.064} & 0.997 & 0.092 & 1.808 & 9.129 & 0.156 & 0.408 & 0.428 \\
        \hline
    \end{tabular}
    }
    \vspace{-4mm}
    
    \label{tab:quantative}
\end{table*}

\begin{figure}[t]
    \centering
    \includegraphics[width=\linewidth]{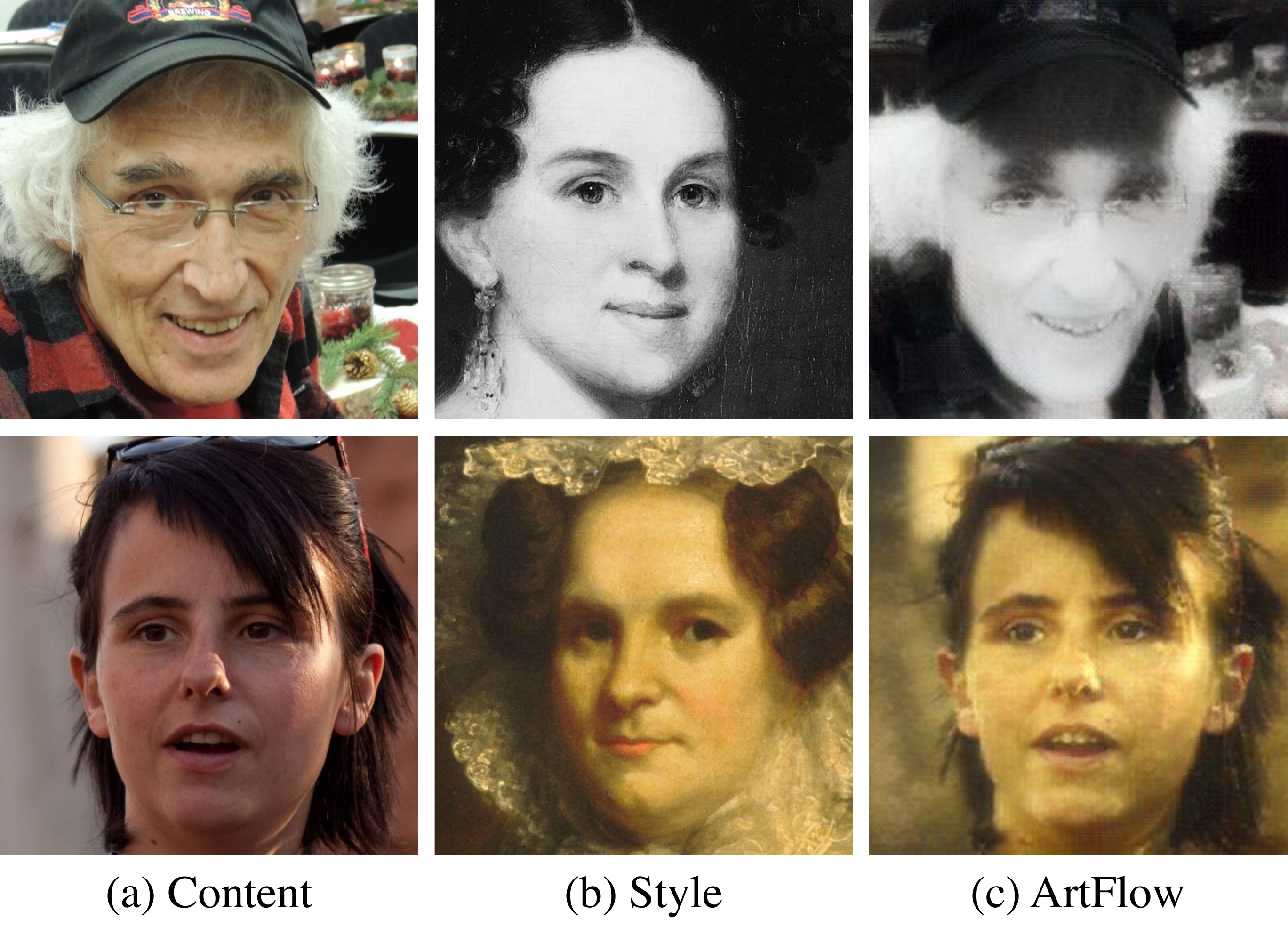}
    \caption{Portrait style transfer results using the proposed ArtFlow.}
    \vspace{-2mm}
    \label{fig:portrait}
\end{figure}

\section{Experiments}
To demonstrate that ArtFlow can achieve unbiased style transfer, we conduct extensive experiments. We make a comparison between the proposed ArtFlow and state-of-the-art style transfer algorithms in terms of stylization effect, computing time, content leak, and content factor visualization. Moreover, we present a new interesting application named reverse style transfer, which can only be performed by ArtFlow. More qualitative results, portrait style transfer images, and user study results are available in supplementary materials.

\subsection{Experimental Settings}

\noindent\textbf{Dataset.} Following the existing style transfer methods \cite{huang2017arbitrary,li2017universal}, we use the MS-COCO dataset \cite{mscoco} as the content images and the WikiArt dataset \cite{wikiart} as the style images. In training, we resize the input images to 512$\times$512 and randomly crop each image to 256$\times$256.

\noindent\textbf{Network architecture.} As shown in Fig. \ref{fig:framework}, the proposed PFN consists of two flow blocks, where each block contains eight neural flows. Each flow is a stack of an Actnorm layer, an invertible 1$\times$1 convolution, and an additive coupling layer. More studies on the number of blocks and flows can be found in the supplementary material.

\noindent\textbf{Training.}
We implement our ArtFlow on the PyTorch framework \cite{pytorch}. We train ArtFlow for 60,000 iterations using an Adam optimizer \cite{kingma2014adam} with a batch size of 2, an initial learning rate of 1e-4, and a learning rate decay of 5e-5. The training of ArtFlow takes about 12 hours on a single RTX 2080Ti GPU. We adopt the content loss $L_c$ in Eq. \ref{eq:content_loss} and style loss $L_s$ in Eq. \ref{eq:style_loss} as the training objective of ArtFlow. The loss weights of $L_c$ and $L_s$ are set to 0.1 and 1, respectively. 

\begin{figure*}
	\centering
	\includegraphics[width=\textwidth]{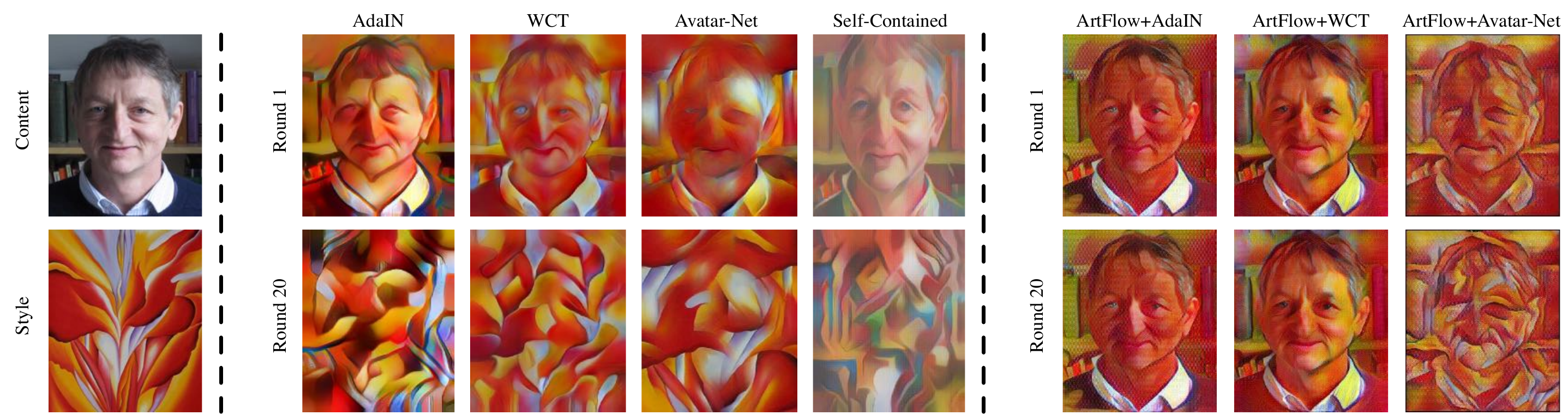}
	\caption{A comparison of the Content Leak phenomenon. We show the style transfer results of the first round and the 20-th round.}
	\vspace{-2mm}
	\label{fig:leak_comp}
\end{figure*}

\subsection{Style Transfer Results}
To demonstrate the style transfer ability of the proposed ArtFlow, we compare the style transfer results of ArtFlow in conjunction with the transfer module of AdaIN/WCT with the  state-of-the-art universal style transfer algorithms, \ie, StyleSwap~\cite{chen2016fast}, AdaIN~\cite{huang2017arbitrary}, WCT~\cite{li2017universal}, LinearWCT~\cite{li2018learning}, OptimalWCT~\cite{lu2019optimal}, and Avatar-Net~\cite{sheng2018avatar}. 

\noindent\textbf{Visual comparison.} Fig.~\ref{fig:result} shows the style transfer results by all the compared algorithms. While all the compared methods can produce good style transfer results, different methods create distinct artistic effects, \eg, WCT and OptimalWCT can create more colorful artistic effects, LinearWCT, AdaIN, ArtFlow can preserve more content details, and Avatar-Net can render more fine textures. The proposed ArtFlow in conjunction with AdaIN/WCT can produce visually comparable style transfer results to the state-of-the-art style transfer algorithms. It is worth noting that the style transfer results by ArtFlow preserve more details of the content image (please zoom in to compare the details of the billboards in the top row results), which confirms that ArtFlow corrects the biased style transfer issue of the compared methods and avoids the Content Leak. Moreover, we also perform portrait style transfer with the proposed ArtFlow. To train the portrait style transfer model, we use FFHQ~\cite{karras2019style} as the content and Metfaces~\cite{karras2020training} as the style. As Fig.~\ref{fig:portrait} shows, ArtFlow can also achieve good artistic style transfer results on portrait images.

\noindent\textbf{Quantitative comparison.} In addition to the visual comparison, we also make a quantitative comparison. Inspired by~\cite{yoo2019photorealistic}, we adopt the Structural Similarity Index (SSIM) and the content loss between the original contents and stylized images as the metric to measure the performance of the content information preservation in style transfer. To measure the ability to create artistic effects of a style transfer algorithm, inspired by~\cite{li2017universal}, we use the mean square error of Gram matrices between the style and the style-transferred images. As Tab.~\ref{tab:quantative} shows, ArtFlow achieves the highest SSIM score, which indicates that the proposed methods have a stronger ability to preserve more content information. While StyleSwap achieves the best content loss and a good SSIM score, its style transfer results do not look as good as the results produced by ArtFlow.
Regarding the Gram loss, since ArtFlow mainly addresses the Content Leak issue and corrects the biases towards style images, it is reasonable that ArtFlow does not achieve the lowest Gram loss. It is worth noting that ArtFlow in conjunction with AdaIN achieves a lower Gram loss than vanilla AdaIN while ArtFlow in conjunction with WCT has a similar Gram loss compared with WCT itself, indicating that ArtFlow can solve the Content Leak issue without hurting the stylization ability of AdaIN/WCT. In the third row of Tab.~\ref{tab:quantative}, we also show the computing time for all the compared methods. ArtFlow+AdaIN is slower than vanilla AdaIN since PFN does not adopt any pooling operations. Therefore, it requires more computations in the higher layers than AdaIN. Comparing with WCT, since ArtFlow does not need the multi-level stylization strategy used by WCT, ArtFlow+WCT is faster than vanilla WCT.

% \noindent\textbf{Training losses.}
% In Fig.~\ref{fig:losses}, we show the training losses of the proposed ArtFlow in comparison with AdaIN. Due to the reversible property of the proposed PFN, our content loss shown in (a) starts from a low value, and the remaining iterations trade-off between the content and style losses. It is worth noting that the converged style loss and the reconstruction error of ArtFlow is significantly lower than AdaIN, which demonstrates that the proposed PFN has a stronger representation ability in making style transfer than existing auto-encoder based framework. Moreover, ArtFlow maintains the image reconstruction error to be zero across the training loop, which reinforces that the proposed PFN can achieve lossless image projection/recovery.

 \begin{table*}[t]
%%
%    \vspace{-1em}
    \caption{User study results of universal stylization methods.}
% 	\vspace{2mm}
    
\centering
%    \scriptsize 
    \footnotesize
%    \small
	\tabcolsep=0.48cm
    \renewcommand{\arraystretch}{0.5}
     \resizebox{1\hsize}{!}{
    \begin{tabular}{lcccccccc}
        \toprule
        Method & StyleSwap & AdaIN & WCT & LinearWCT & OptimalWCT & Avatar-Net & ArtFlow \\
        \midrule
        Votes~$\uparrow$ & 7 & 19 & 64 & 257 & 60 & 78 & \textbf{314} \\
        \bottomrule
    \end{tabular}
    }
    \vspace{-4mm}
    
    \label{tab:user_study}
\end{table*}

\begin{figure}
    \centering
    \includegraphics[width=\linewidth]{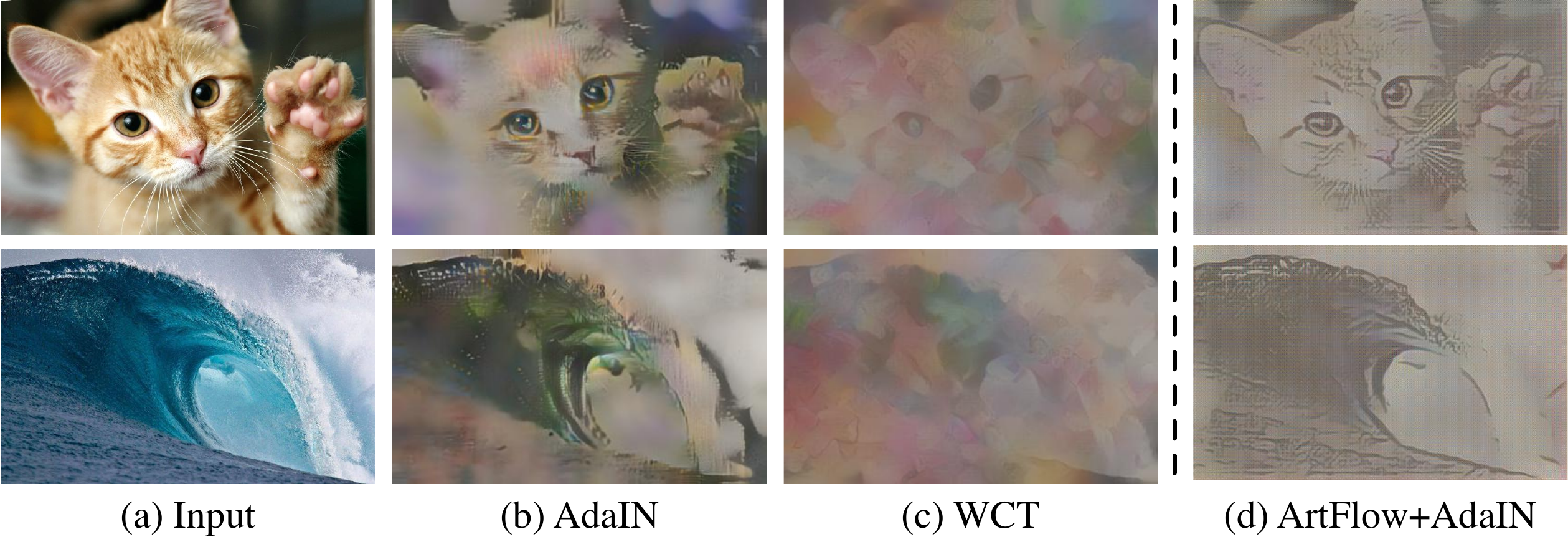}
    \caption{Visualization of content features of AdaIN, WCT, and the proposed ArtFlow.}
    \vspace{-4mm}
    \label{fig:content_visualize}
\end{figure}

\subsection{Content Leak}
%\vspace{-2mm}
As discussed in Sec.~\ref{sec:analysis}, if the Content Leak happens, the content information would gradually disintegrate when we perform style transfer repeatedly.
To demonstrate that the proposed ArtFlow can avoid the Content Leak phenomenon, we use the above way to visualize and compare the Content Leak phenomenon in AdaIN, WCT, Avatar-Net, and their counterparts in conjunction with the proposed ArtFlow. We also show the result by~\cite{chen2020self} because it also addresses the content leak issue. As Fig.~\ref{fig:leak_comp} shows, the Content Leak appears in vanilla AdaIN, WCT, Avatar-Net, and Self-Content~\cite{chen2020self} when we perform the style transfer for 20 rounds. In contrast, when we replace the VGG19 based auto-encoder with the proposed PFN in AdaIN/WCT, the Content Leak disappears completely, which indicates that ArtFlow in conjunction with AdaIN/WCT can effectively solve the Content Leak problem and therefore achieve unbiased style transfer. Regarding Avatar-Net, as discussed in Sec.~\ref{causes}, since the Style Decorator in Avatar-Net is inherently biased towards style, ArtFlow combining the Style Decorator as the transfer module cannot achieve unbiased style transfer. However, by replacing the auto-encoder with PFN, the Content Leak phenomenon is still significantly alleviated by Avatar-Net.

%\vspace{-2mm}
\subsection{Content-Style Separation}
%\vspace{-1mm}
As discussed in Sec.~\ref{sec:separation}, AdaIN and WCT can be regarded as an evolution of the Bilinear Model in~\cite{Tenenbaum2000}. Taking the view of the bilinear model, the mechanism of AdaIN and WCT can be regarded as: 1) disentangling the content and style factors in the deep feature space, and 2) replacing the style factor of the content image with the style factor of the style image. Since such a disentangled representation of the content and style exists in the feature space, we can visualize the pure content by inverting the content factor back to an image. Fig.~\ref{fig:content_visualize} shows the inverted content factor in AdaIN, WCT, and ArtFlow in conjunction with AdaIN. Compared with the inverted content factor of AdaIN and WCT, the results by ArtFlow contain significantly less style effects (\eg, colors) along with sharper image structures. Fig.~\ref{fig:content_visualize} shows that ArtFlow can achieve unbiased content-style separation while AdaIN and WCT cannot.

\section{User Study}
 To quantitatively demonstrate that the proposed ArtFlow has the comparable style transfer performance with the state-of-the-art algorithms, we perform a user study. Our user study is based on the validation dataset that consists of 43 content images and 27 style images. We obtain the style transfer results of StyleSwap, AdaIN, WCT, LinearWCT, OptimalWCT, Avatar-Net, and the proposed ArtFlow on every content-style pair, respectively. We finally obtain 1161 style transfer results for each method. In user study, we list all style transfer results of a content-style pair and let the user to choose ONE most preferable style transfer result. We eventually collect 799 effective votes. Tab.~\ref{tab:user_study} shows the style transfer results. The proposed ArtFlow obtains more votes compared with other style transfer methods, which demonstrates that our method has comparable style transfer performance with the state-of-the-art methods.

%\subsection{Reverse Style Transfer}
%Because the proposed PFN can achieve lossless and unbiased image projection and reversion, ArtFlow enables a very interesting application, where if we perform style transfer on a style-transferred image with the original content image as the new style image, we can recover the original content image in a lossless manner. We name this application \emph{Reverse Style Transfer}. Fig.~\ref{fig:reverse} shows the reverse style transfer results by AdaIN, WCT, and the proposed ArtFlow. It is remarkable that only the proposed ArtFlow can reverse the style transfer process and recover all the details of the original content image.
%
%\begin{figure}
%    \centering
%    \includegraphics[width=\linewidth]{./figures/reverse}
%    \caption{A comparison of reverse style transfer results.}
%%    \vspace{-2mm}
%    \label{fig:reverse}
%\end{figure}

%\vspace{-2mm}
\section{Conclusion}
%\vspace{-1mm}
In this paper, we reveal a common issue in the state-of-the-art style transfer algorithms, \ie, the Content Leak phenomenon. Upon analyzing the main causes of the Content Leak, we present a new style transfer framework named ArtFlow. Unlike the existing style transfer algorithms, which adopt the VGG19 based auto-encoder to extract deep features, ArtFlow introduces a reversible neural flow-based network named PFN, thus enabling both the forward and reverse inferences to project images into the feature space and invert features back to the image space, respectively. ArtFlow in conjunction with an unbiased style transfer module, \eg, either AdaIN or WCT, achieves comparable style transfer results while avoiding the Content Leak phenomenon. Furthermore, because PFN can achieve lossless and unbiased image projection and reversion, the proposed ArtFlow can facilitate a better content-style separation and thus enable the reversion of the style transfer in a lossless manner.

\section{Acknowledgement}
%\vspace{-1mm}
This work is supported in part by NSF awards IIS-1704337, IIS-1813709, and our corporate sponsors.
\clearpage

{\small
\bibliographystyle{ieee_fullname}
\bibliography{main}
}
\clearpage

\appendix
\section{Pseudocode of ArtFlow}
\begin{algorithm}
	\caption{Unbiased style transfer process of ArtFlow}
	\begin{algorithmic}
	\REQUIRE an input content image $I_c$, an input style image $I_s$, an unbiased style transfer module $T$, the proposed new network PFN;
	\STATE - feed $I_c$ and $I_s$ to PFN;
	\STATE - perform the forward propagation of PFN and obtain the content feature $f_c$ and the style feature $f_s$, respectively;
	\STATE - obtain the style-transferred feature $f_{cs}$ based on $T$ in the manner of $f_{cs} = T(f_c, f_s)$;
	\STATE obtain the style-transferred image $I_{cs}$ by running the reverse propagation of PFN.
	\RETURN $I_{cs}$
	\end{algorithmic}
	\end{algorithm}

\section{Proof of the Theorem 1}
\begin{theorem}
\label{theorem1}
The adaptive instance normalization in AdaIN is an unbiased style transfer module.
\end{theorem}
 \begin{proof}
 Without loss of generality, we assume both $f_c$ and $f_s$ is centered. Therefore, we have,
 \begin{equation}
     f_{cs} = \frac{f_c}{\sigma(f_c)} \sigma(f_s),
 \end{equation}
 where,
 \begin{equation}
     C(f) = \frac{f}{\sigma(f)},\ S(f) = \sigma(f).
 \end{equation}
 Since,
 \begin{equation}
     \sigma(f_{cs}) = \sigma(\frac{f_c}{\sigma(f_c)}) \sigma(f_s) = 1 \cdot \sigma(f_s) = \sigma(f_s),
 \end{equation}
 we have,
 \begin{align}
     C(f_{cs}) &= \frac{f_{cs}}{\sigma(f_{cs})} = \frac{f_c}{\sigma(f_c)} = C(f_c), \\
     S(f_{cs}) &= \sigma(f_{cs}) = \sigma(f_s).
 \end{align}
 Therefore, the adaptive instance normalization in AdaIN is unbiased. $\hfill\blacksquare$ 
 \end{proof}
 
 \section{Proof of the Theorem 2}
 \begin{theorem}
\label{theorem2}
The whitening and coloring transforms in WCT is an unbiased style transfer module.
\end{theorem}
 \begin{proof}
 Without loss of generality, we assume both $f_c$ and $f_s$ is centered. Therefore, we have,
 \begin{align}
     &\mathrm{Whitening:} \hat{f}_c = E_c D_c^{-1/2} E_c^T f_c,\\
     &\mathrm{Coloring:} f_{cs} = E_s D_s^{1/2} E_s^T \hat{f}_c.
 \end{align}
 where $f_c, f_{cs}$ represent the content and style-transferred features, $f_c f_c^T = E_c D_c E_c^T$ and $f_s f_s^T = E_s D_s E_s^T$. Therefore, the style-transferred features $f_{cs}$ can be expressed as,
 \begin{equation}
     f_{cs} = E_s D_s^{1/2} E_s^T E_c D_c^{-1/2} E_c^T f_c.
 \end{equation}
 Here
 \begin{equation}
     C(f) = E D^{-1/2} E^T f, \ S(f) = E D^{1/2} E^T
 \end{equation}
 Since,
 \begin{equation}
     f_{cs}f_{cs}^T = f_s f_s^T = E_s D_s E_s^T,
 \end{equation}
 we have,
 \begin{align}
     C(f_{cs}) &= E_s D_s^{-1/2}E_s^T f_{cs} \\
     &= E_c D_c^{-1/2} E_c^T f_c \\
     &= C(f_c), \\
     S(f_{cs}) &= E_s D_s^{1/2}E_s^T = S(f_s).
     \label{eq:biased_transfer}
 \end{align}
 Therefore, the whitening and coloring transforms in WCT is unbiased. $\hfill\blacksquare$ 
 \end{proof}
 
 \begin{figure*}[t]
	\centering
	\includegraphics[width=\textwidth]{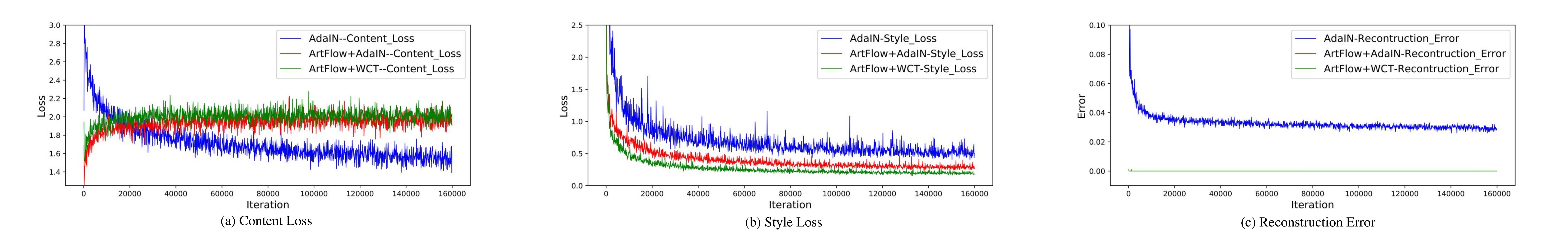}
	\caption{A comparison of the training losses of AdaIN and the proposed ArtFlow.}
	\vspace{20mm}
	\label{fig:losses}
\end{figure*}

 \begin{figure*}[t]
	\centering
	\includegraphics[width=\textwidth]{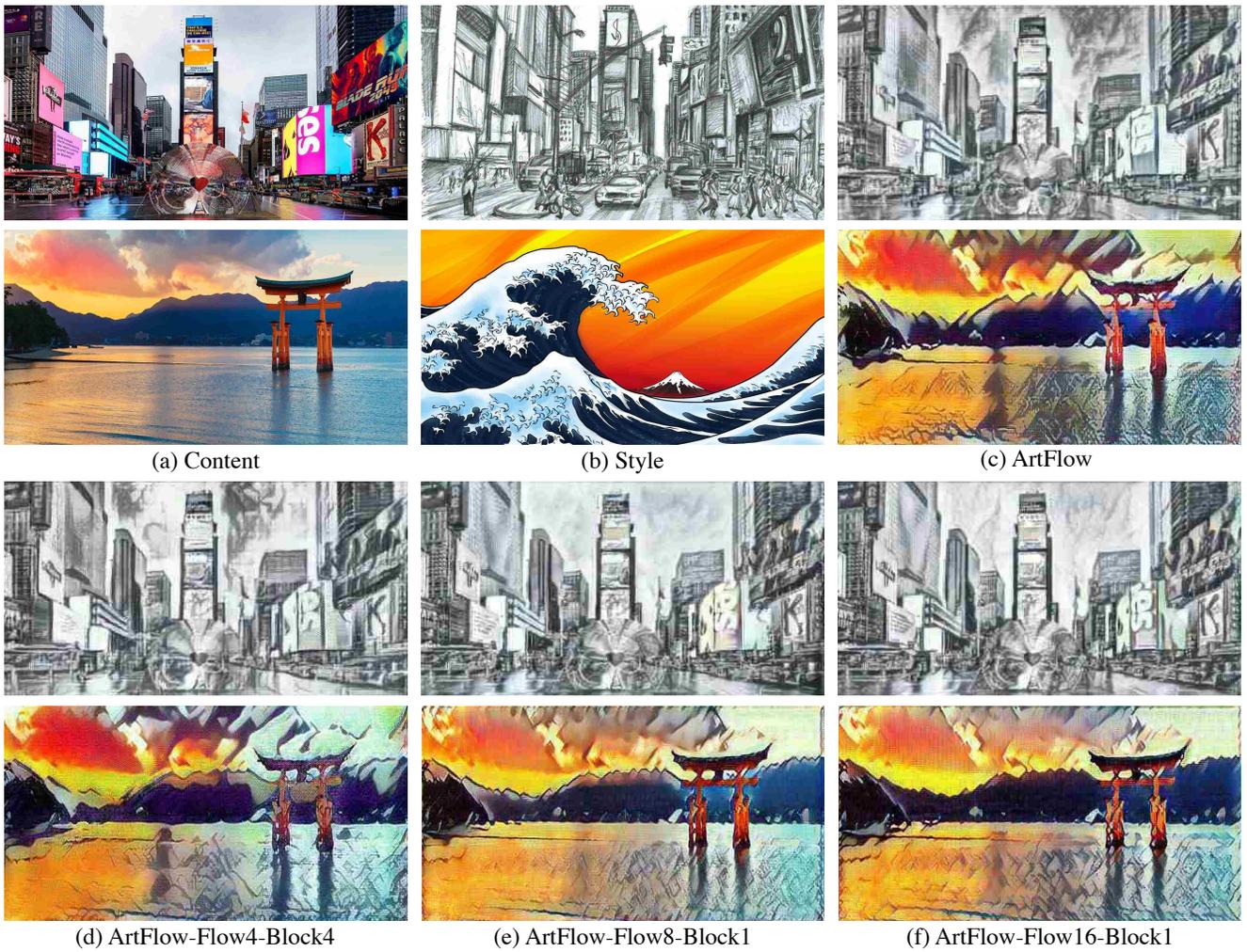}
	\caption{A comparison of the style transfer result by different neural flow architectures.}
	\vspace{20mm}
	\label{fig:ablation}
\end{figure*}

\begin{table*}[t]
%%
%    \vspace{-1em}
    \caption{Quantitative evaluation results different flow architectures.}
	\vspace{2mm}
    
\centering
    %\scriptsize 
    \small
     \tabcolsep=0.3cm
     \renewcommand{\arraystretch}{0.9}
    \begin{tabular}{lcccc}
        \toprule
        Method & ArtFlow (PFN)  & ArtFlow-Flow4-Block4 & ArtFlow-Flow8-Block1 & ArtFlow-Flow16-Block1 \\
        \midrule
        SSIM~$\uparrow$ & 0.45 & 0.401 & \textbf{0.482} & 0.480 \\
        Gram Loss~$\downarrow$ & \textbf{0.00098} & 0.00130 & 0.00139 & 0.00137 \\
        \bottomrule
    \end{tabular}
%    \vspace{-2mm}
    
    \label{tab:quantitative}
\end{table*}
% \begin{table*}[t]
%%%
%%    \vspace{-1em}
%    \caption{User study results of universal stylization methods.}
% 	\vspace{2mm}
%    
%\centering
%    %\scriptsize 
%    % \footnotesize
%    \small
%	\tabcolsep=0.38cm
%    \renewcommand{\arraystretch}{0.9}
%%     \resizebox{1\hsize}{!}{
%    \begin{tabular}{lcccccccc}
%        \toprule
%        Method & StyleSwap & AdaIN & WCT & LinearWCT & OptimalWCT & Avatar-Net & ArtFlow \\
%        \midrule
%        Votes~$\uparrow$ & 7 & 19 & 64 & 257 & 60 & 78 & \textbf{314} \\
%        \bottomrule
%    \end{tabular}
%%    }
%%    \vspace{-4mm}
%    
%    \label{tab:user_study}
%\end{table*}
\begin{figure*}
    \centering
    \includegraphics[width=0.9\linewidth]{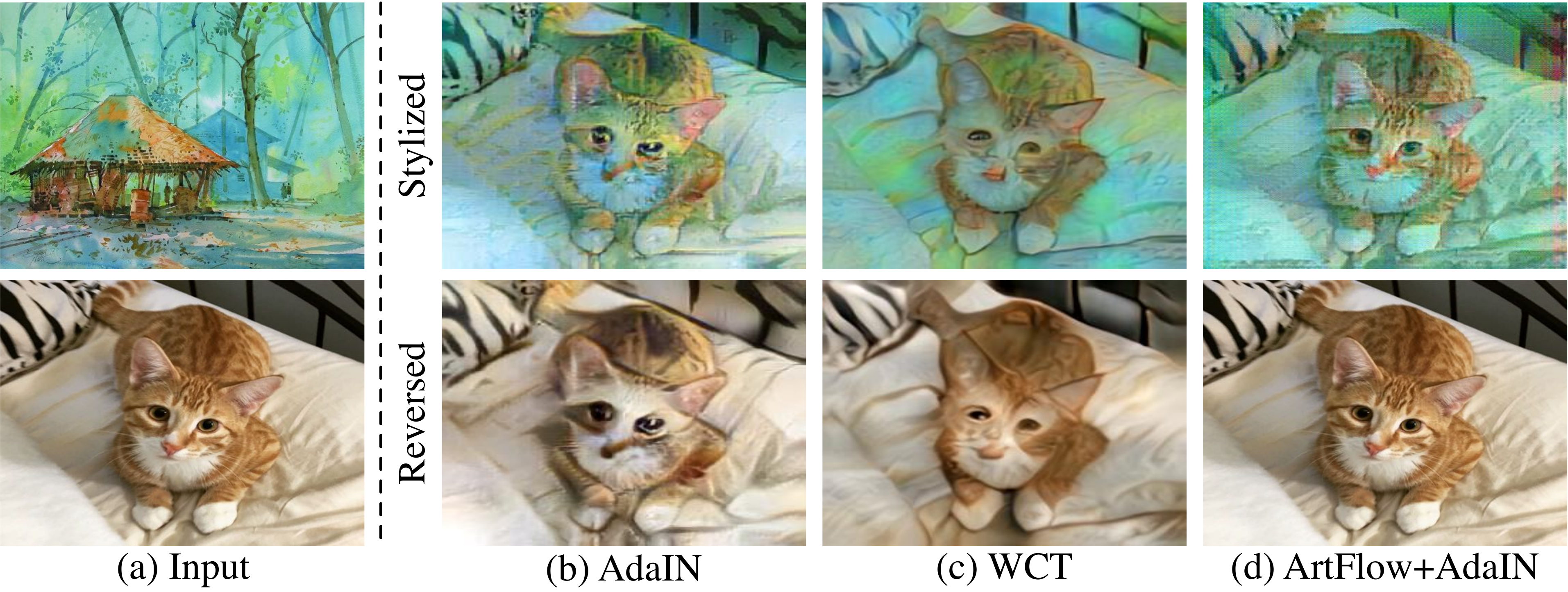}
    \caption{A comparison of reverse style transfer results.}
    \vspace{-2mm}
    \label{fig:reverse}
\end{figure*}
 
 \section{Training Loss Trends}
 In Fig.~\ref{fig:losses}, we show the training losses of the proposed ArtFlow in comparison with AdaIN. Due to the reversible property of the proposed PFN, our content loss shown in (a) starts from a low value, and the remaining iterations trade-off between the content and style losses. It is worth noting that the converged style loss and the reconstruction error of ArtFlow is significantly lower than AdaIN, which demonstrates that the proposed PFN has a stronger representation ability in making style transfer than existing auto-encoder based framework. Moreover, ArtFlow maintains the image reconstruction error to be zero across the training loop, which reinforces that the proposed PFN can achieve lossless image projection/recovery.
 
 \section{Network Structure Analysis of PFN}
 In addition to the proposed PFN, we also conduct experiments on three alternative network architectures based on neural flows: ArtFlow-Flow8-Block1, ArtFlow-Flow16-Block1, and ArtFlow-Flow4-Block4, where the number after ``Flow'' represents the number of cascades flow modules within a flow block, while the number after ``Block'' denotes the number of flow blocks. In the experimented network, each flow is a stack of an Actnorm layer, an invertible 1$\times$1 convolution, and an additive coupling layer. According to this naming scheme, the proposed PFN is ArtFlow-Flow8-Block2. Fig.~\ref{fig:ablation} shows the style transfer results by the proposed PFN and all the compared architectures. As Fig.~\ref{fig:ablation} (e) and (f) show, the style transfer results by ArtFlow-Flow16-Block1 and ArtFlow-Flow16-Block1 are not as colorful as other compared architectures, which indicates that these two architectures have compromised stylization abilities. Regarding ArtFlow-Flow4-Block4, it can produce comparable style transfer results to the proposed PFN. However, in terms of the details, the proposed PFN is better than ArtFlow-Flow4-Block4. Please zoom in to see the details of the Itsukushima Shrine, where the proposed PFN creates more colorful results. We also conduct a quantitative comparison between the proposed ArtFlow and three alternative architectures. As Tab.~\ref{tab:quantitative} shows, the proposed PFN achieves a better trade-off between the SSIM index and the Gram loss.
 
 \subsection{Reverse Style Transfer}
Because the proposed PFN can achieve lossless and unbiased image projection and reversion, ArtFlow enables a very interesting application, where if we perform style transfer on a style-transferred image with the original content image as the new style image, we can recover the original content image in a lossless manner. We name this application the \emph{Reverse Style Transfer}. Fig.~\ref{fig:reverse} shows the reverse style transfer results by AdaIN, WCT, and the proposed ArtFlow. It is remarkable that only the proposed ArtFlow can reverse the style transfer process and recover all the details of the original content image.
 
 \section{More Style Transfer Results}
 We show more comparisons of style transfer results between the proposed ArtFlow and state-of-the-art algorithms in Fig.~\ref{fig:comp1}, Fig.~\ref{fig:comp2}, Fig.~\ref{fig:comp3}, Fig.~\ref{fig:comp4}, and Fig.~\ref{fig:comp5}.
 
  \section{Portrait Style Transfer Results}
 We use the proposed ArtFlow to make style transfer on portrait images. To train the portrait style transfer model, we use images of the FFHQ~\cite{karras2019style} dataset as the content and images of the Metfaces~\cite{karras2020training} dataset as the style. Fig.~\ref{fig:portrait1} and Fig.~\ref{fig:portrait2} show the portrait style transfer results by the proposed ArtFlow.
 
  \section{More Content Leak Results}
 We show more comparisons of the \emph{Content Leak} phenomenon in Fig.~\ref{fig:leak1} and Fig.~\ref{fig:leak2}.
 
   \section{More Content Factor Reconstruction Results}
 We show more comparisons of the content factor reconstruction results in Fig.~\ref{fig:separation}.
 
% \section{User Study}
% To quantitatively demonstrate the comparable style transfer performance of the proposed ArtFlow compared with the state-of-the-art style transfer algorithms. We perform a user study and let users subjectively measure the style transfer effects of all the compared style transfer algorithms. Our user study is based on the validation dataset that consists of 43 content images and 27 style images. We get the style transfer results by using StyleSwap, AdaIN, WCT, LinearWCT, OptimalWCT, Avatar-Net, and the proposed ArtFlow on every content-style pair, respectively. We finally obtain 1161 style transfer results for each method. In user study, we list all style transfer results with the same content-style pair by all the compared methods and let the user to choose ONE most preferable style transfer result. We eventually collect xxx effective votes. Tab.~\ref{tab:user_study} shows the style transfer results. The proposed ArtFlow obtains more votes compared with other style transfer methods, which demonstrates that our method has comparable style transfer performance with the state-of-the-art methods.
 
 \begin{figure*}[t]
	\centering
	\includegraphics[width=\textwidth]{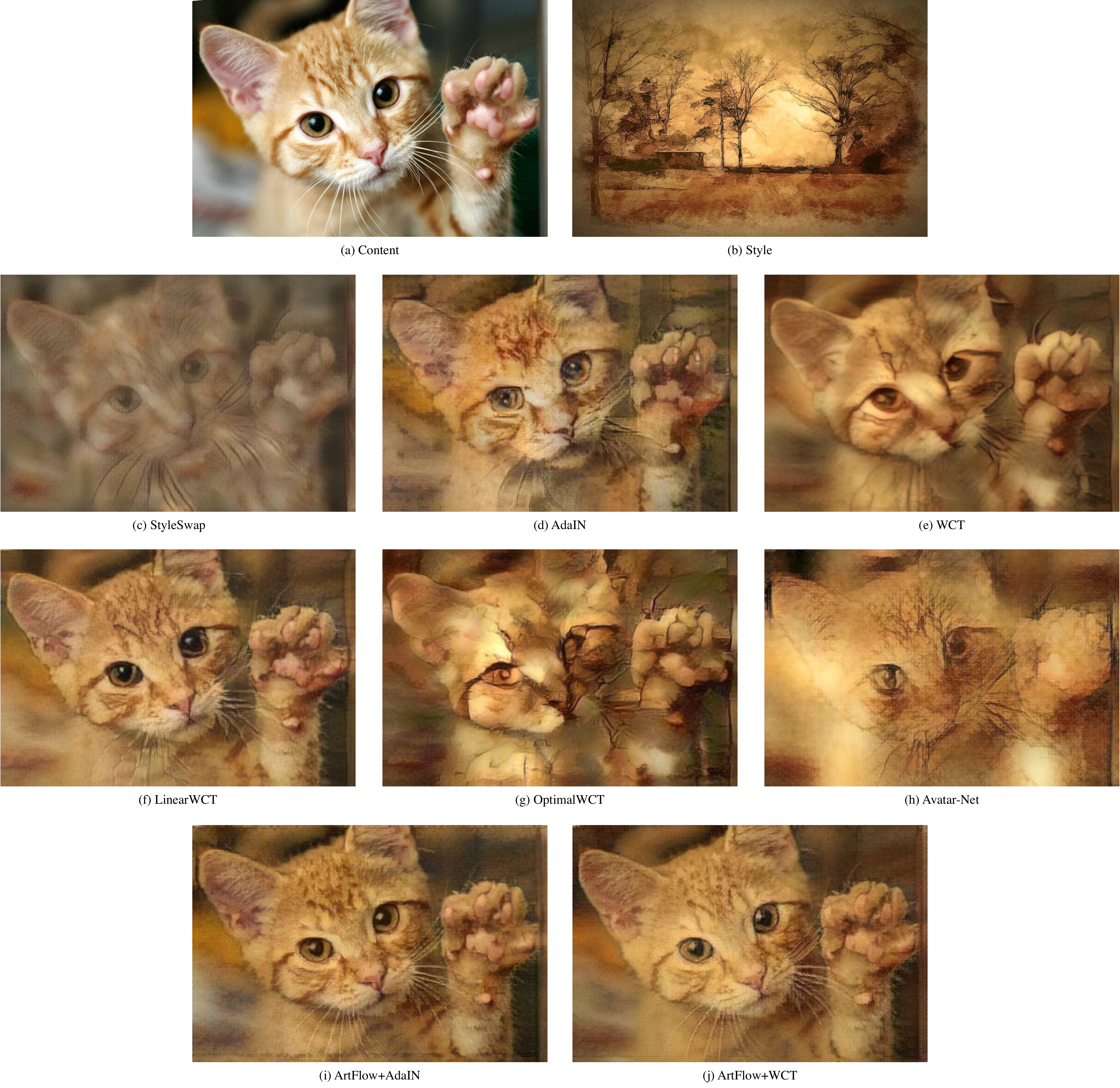}
	\caption{A comparison of the style transfer results between the proposed ArtFlow and state-of-the-art algorithms.}
%	\vspace{20mm}
	\label{fig:comp1}
\end{figure*}

\begin{figure*}[t]
	\centering
	\includegraphics[width=\textwidth]{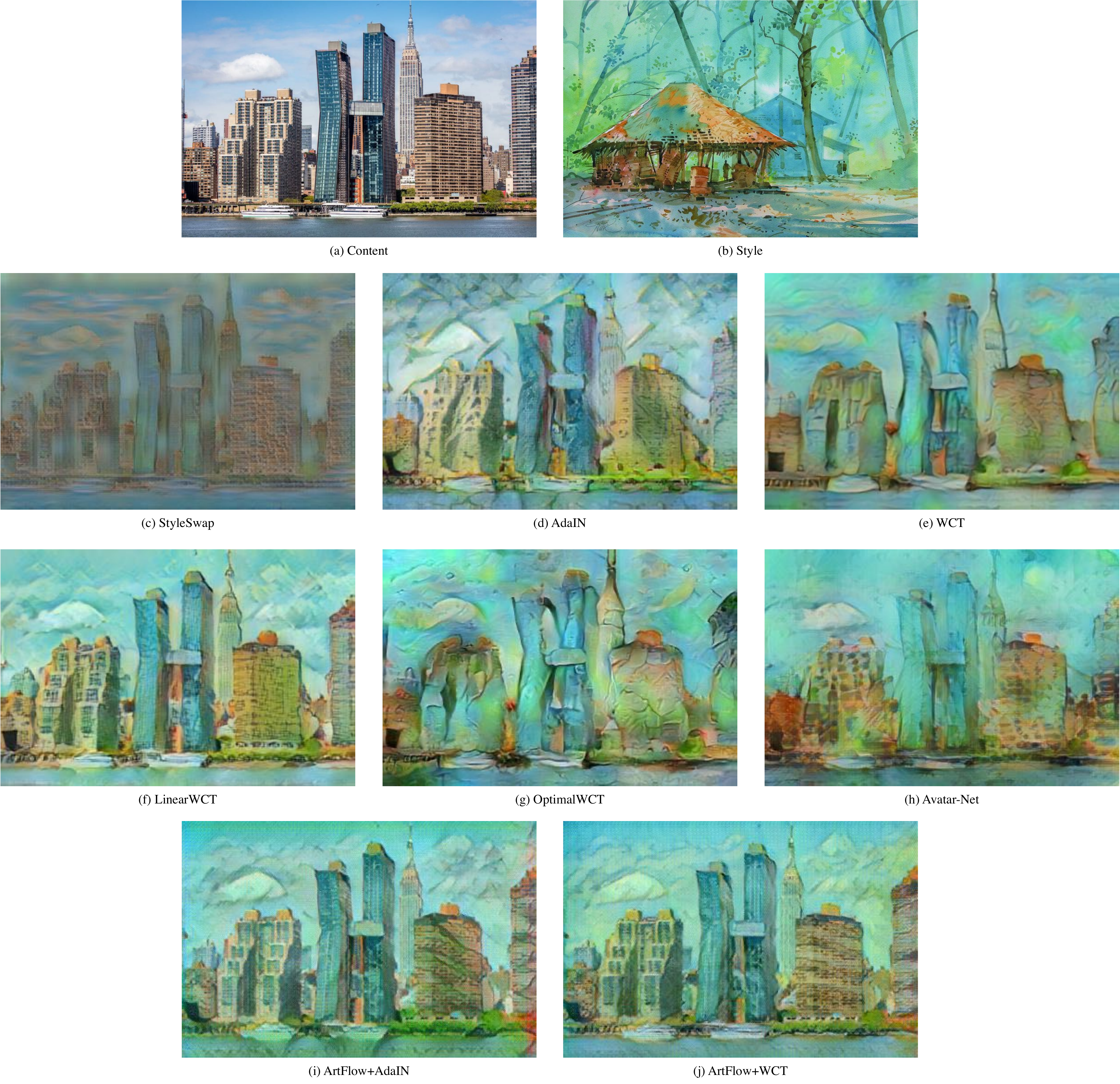}
	\caption{A comparison of the style transfer results between the proposed ArtFlow and state-of-the-art algorithms.}
%	\vspace{20mm}
	\label{fig:comp2}
\end{figure*}

\begin{figure*}[t]
	\centering
	\includegraphics[width=\textwidth]{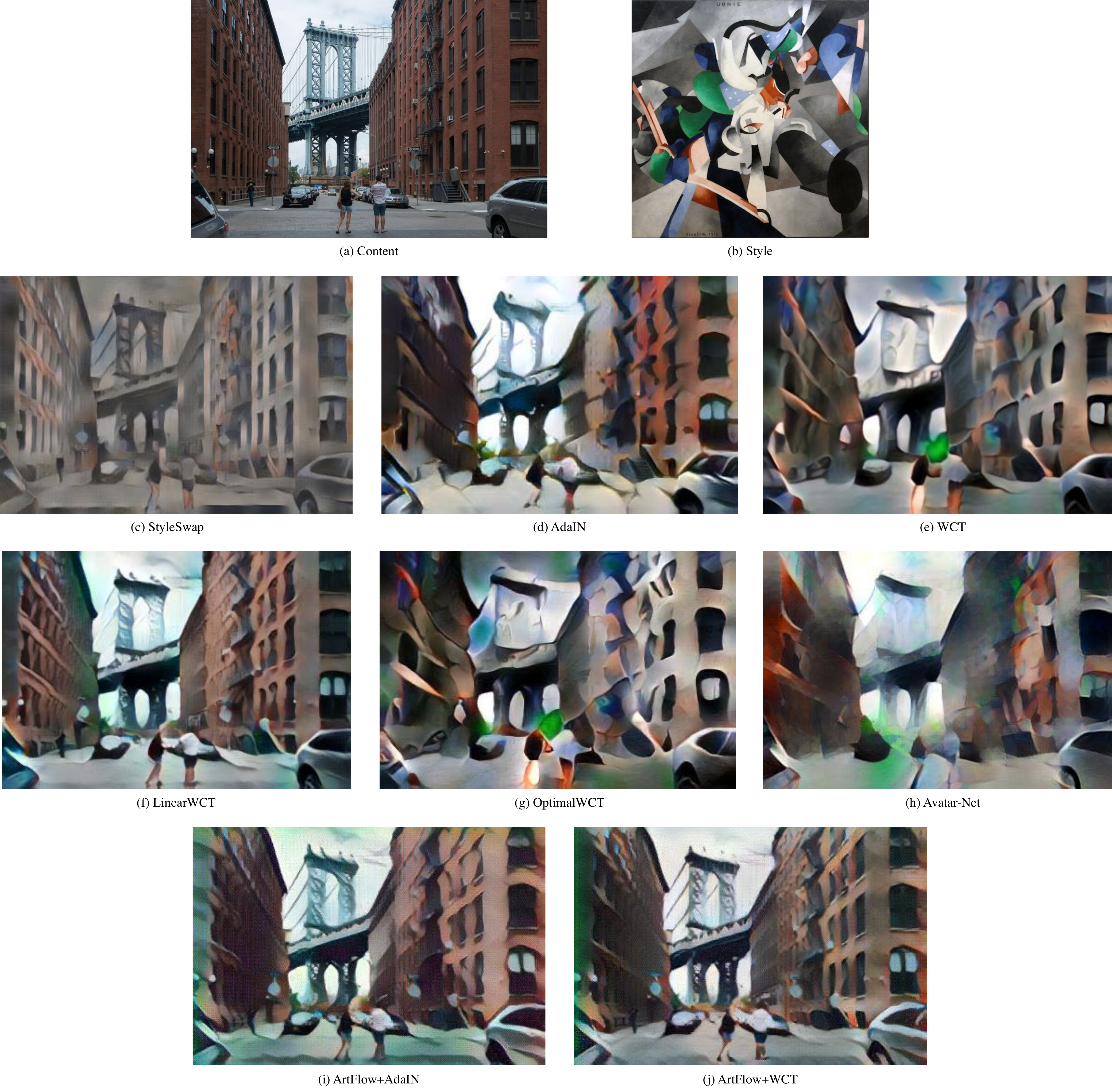}
	\caption{A comparison of the style transfer results between the proposed ArtFlow and state-of-the-art algorithms.}
%	\vspace{20mm}
	\label{fig:comp3}
\end{figure*}

\begin{figure*}[t]
	\centering
	\includegraphics[width=\textwidth]{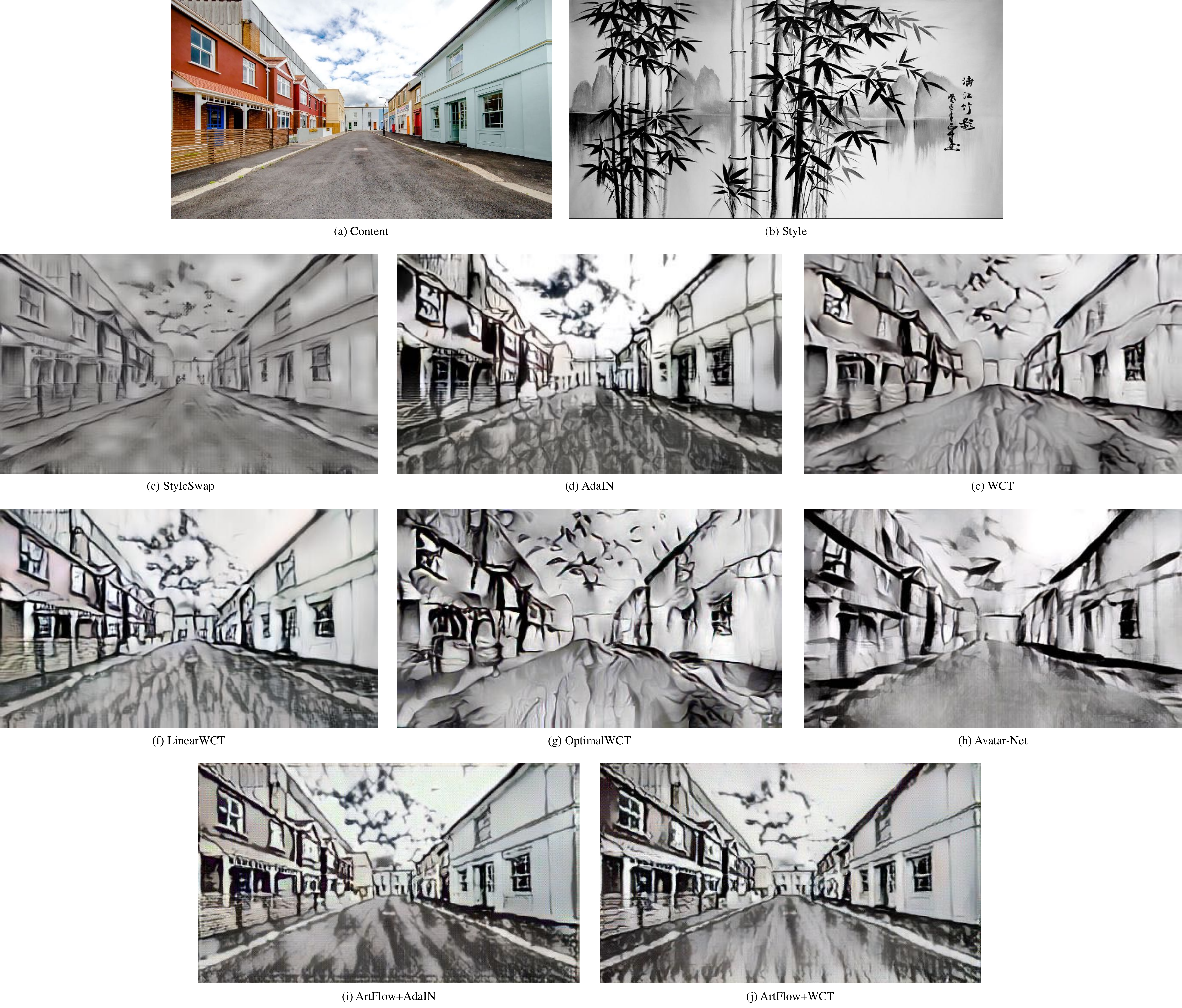}
	\caption{A comparison of the style transfer results between the proposed ArtFlow and state-of-the-art algorithms.}
%	\vspace{20mm}
	\label{fig:comp4}
\end{figure*}

\begin{figure*}[t]
	\centering
	\includegraphics[width=\textwidth]{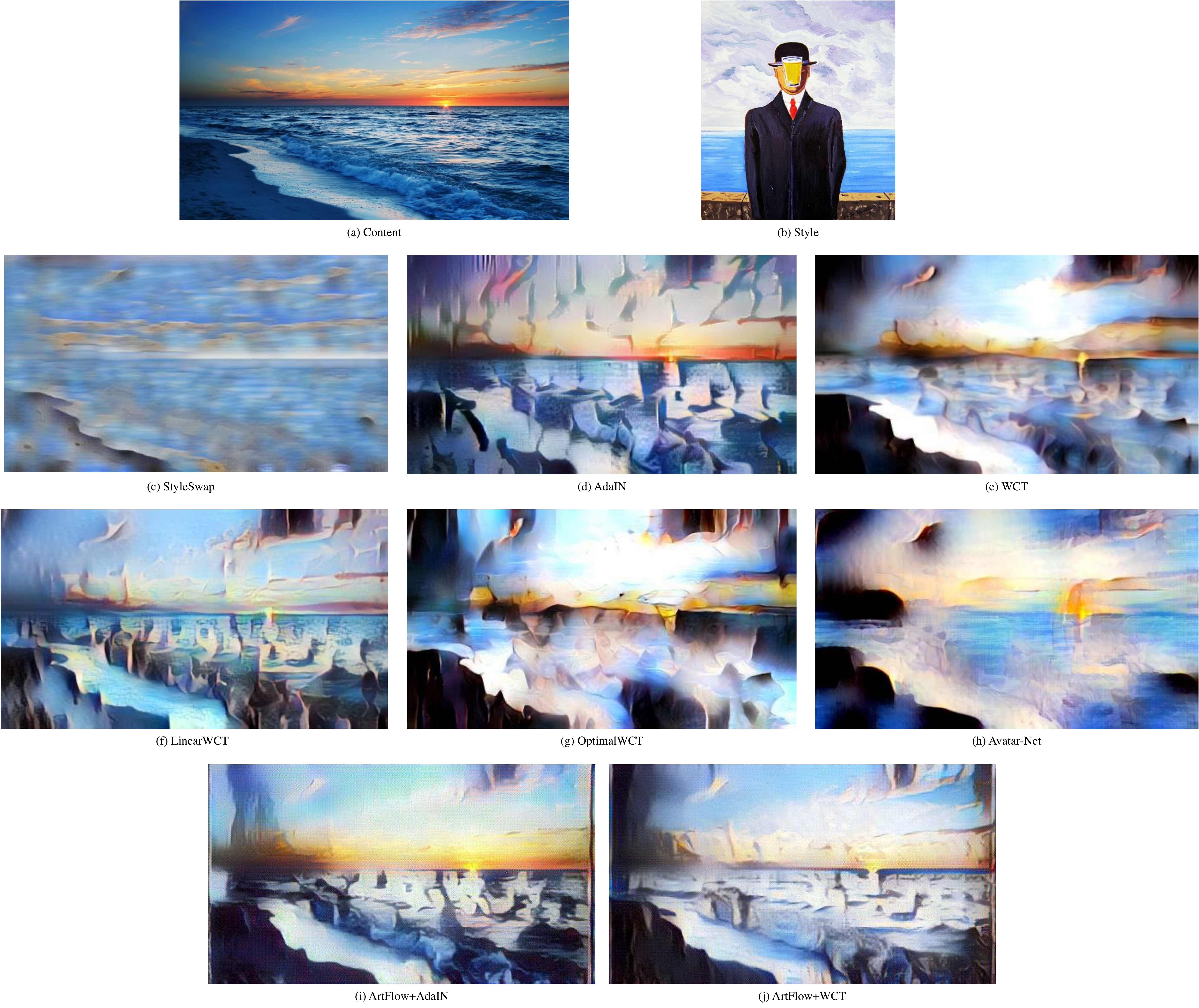}
	\caption{A comparison of the style transfer results between the proposed ArtFlow and state-of-the-art algorithms.}
%	\vspace{20mm}
	\label{fig:comp5}
\end{figure*}

\begin{figure*}[t]
	\centering
	\includegraphics[width=0.85\textwidth]{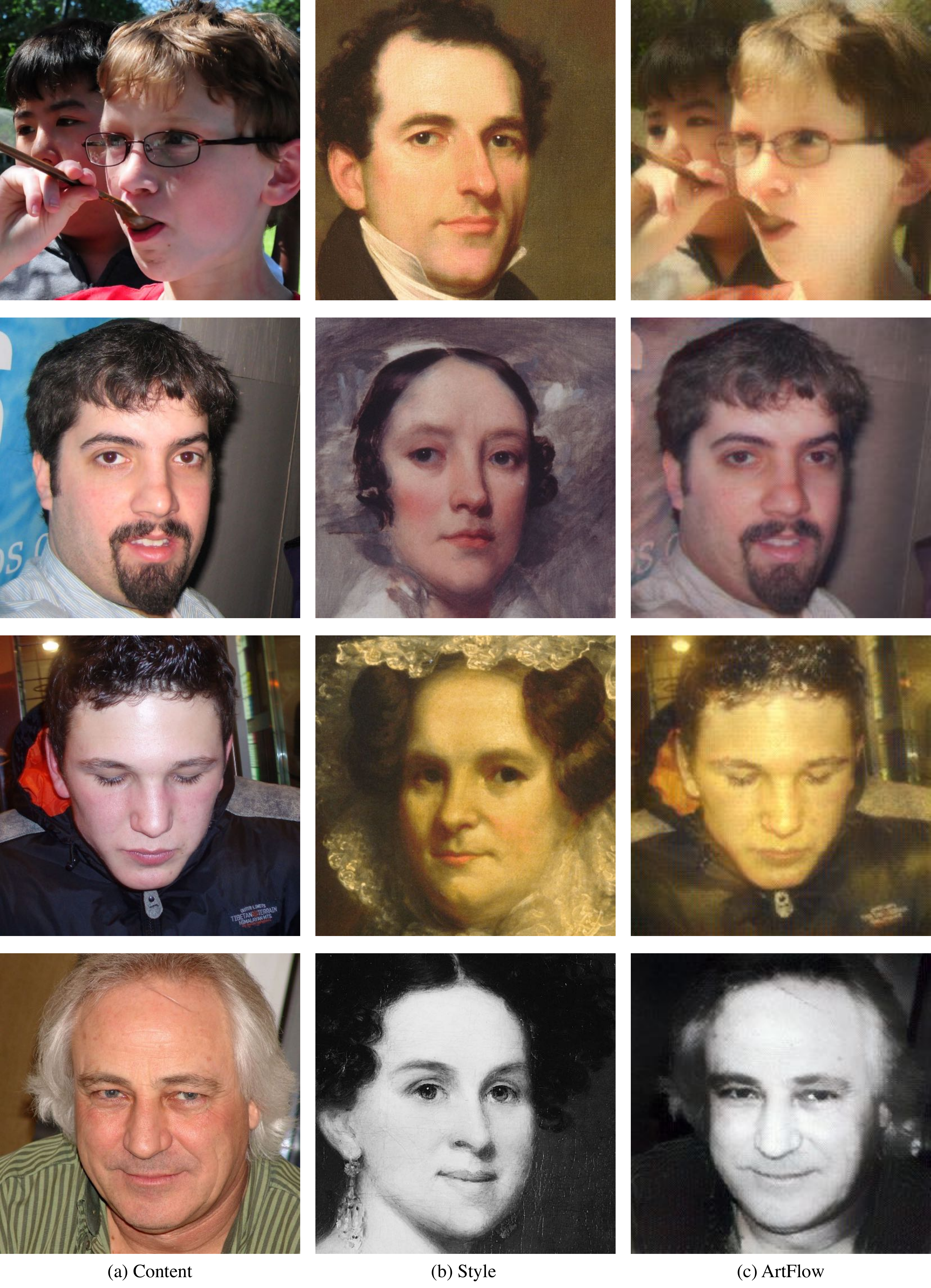}
	\caption{Portrait style transfer results by the proposed ArtFlow.}
%	\vspace{20mm}
	\label{fig:portrait1}
\end{figure*}

\begin{figure*}[t]
	\centering
	\includegraphics[width=0.85\textwidth]{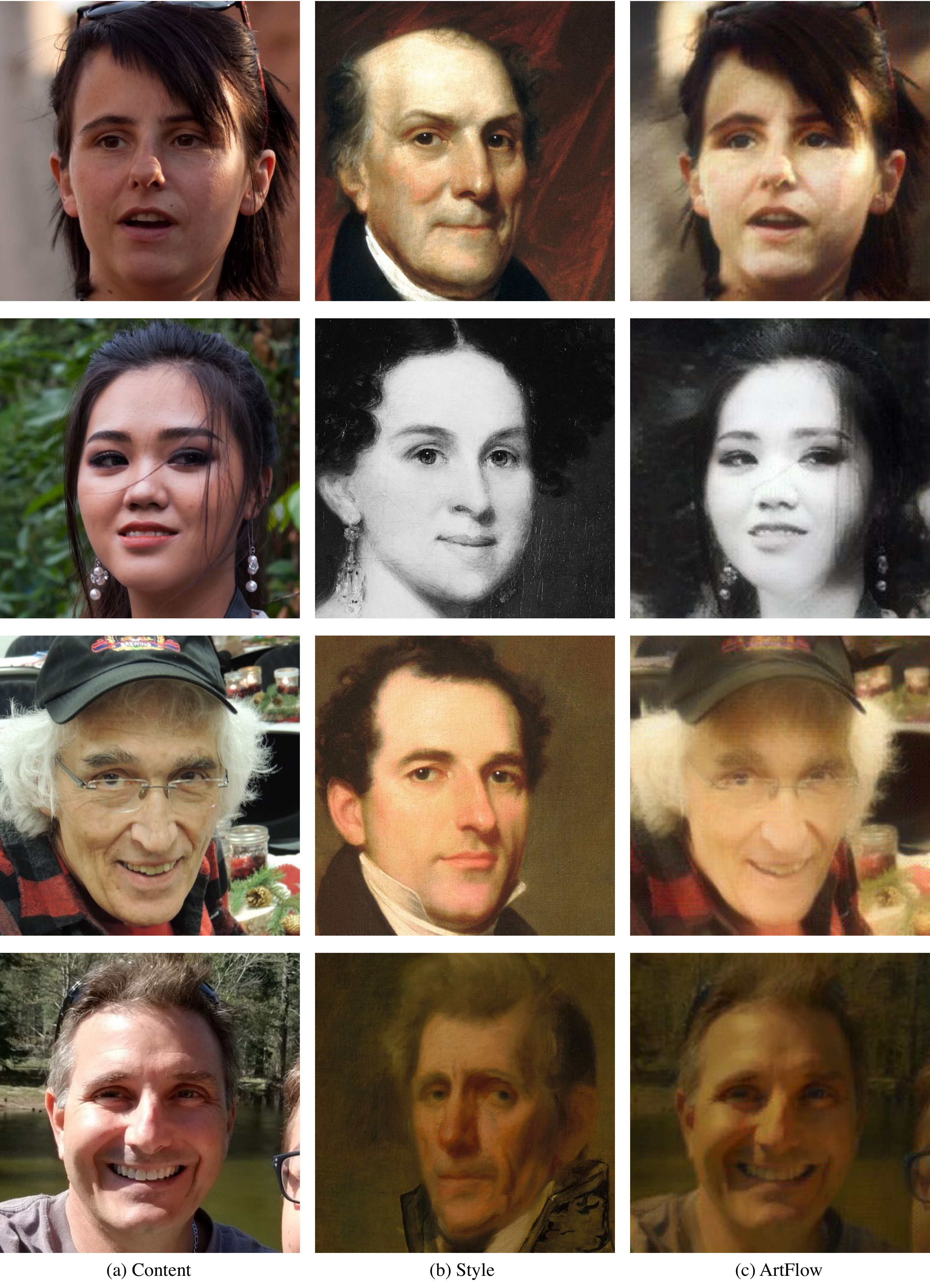}
	\caption{Portrait style transfer results by the proposed ArtFlow.}
%	\vspace{20mm}
	\label{fig:portrait2}
\end{figure*}

\begin{figure*}[t]
	\centering
	\includegraphics[width=\textwidth]{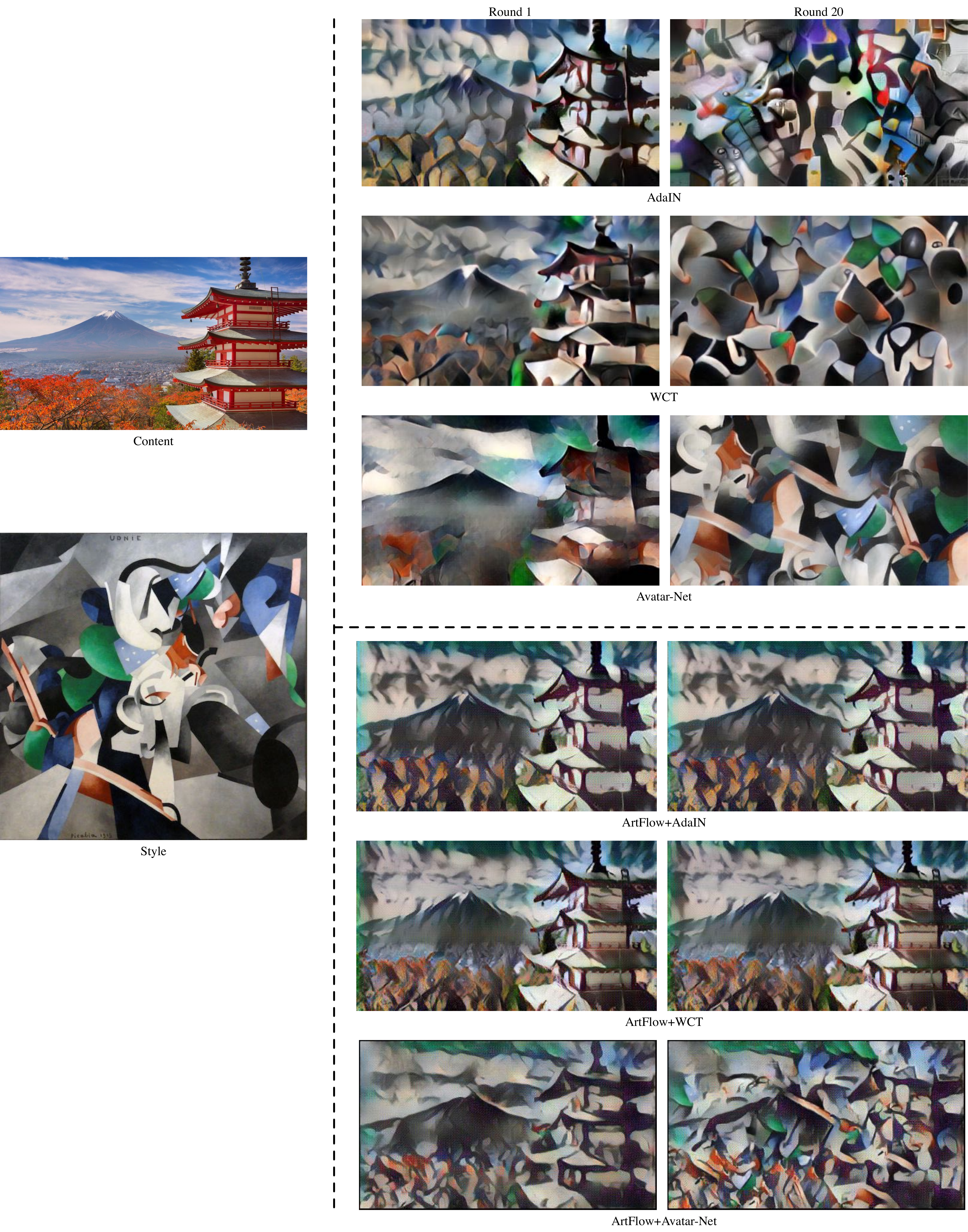}
	\caption{A comparison of the \emph{Content Leak} phenomenon.}
%	\vspace{20mm}
	\label{fig:leak1}
\end{figure*}

\begin{figure*}[t]
	\centering
	\includegraphics[width=\textwidth]{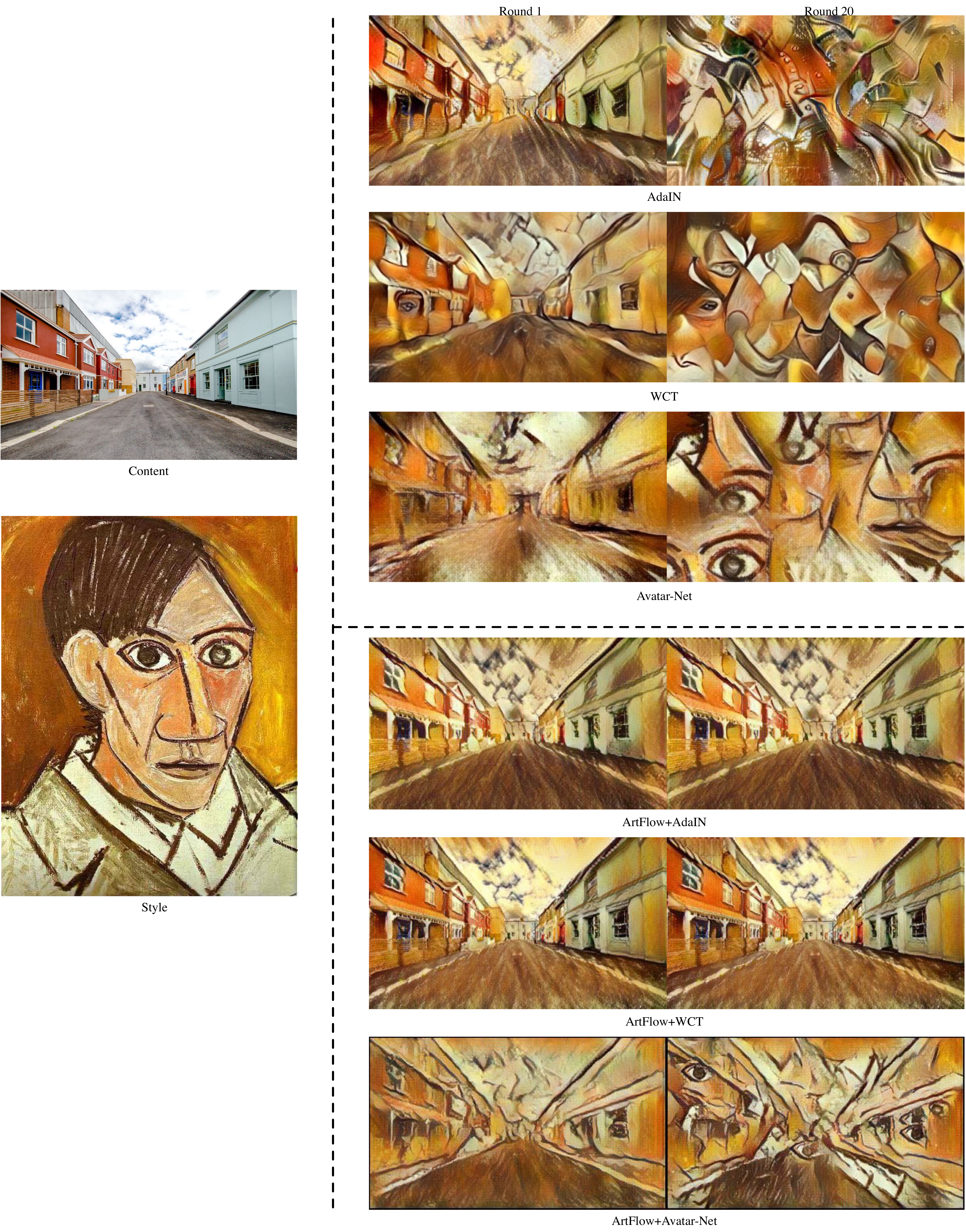}
	\caption{A comparison of the \emph{Content Leak} phenomenon.}
%	\vspace{20mm}
	\label{fig:leak2}
\end{figure*}

\begin{figure*}[t]
	\centering
	\includegraphics[width=\textwidth]{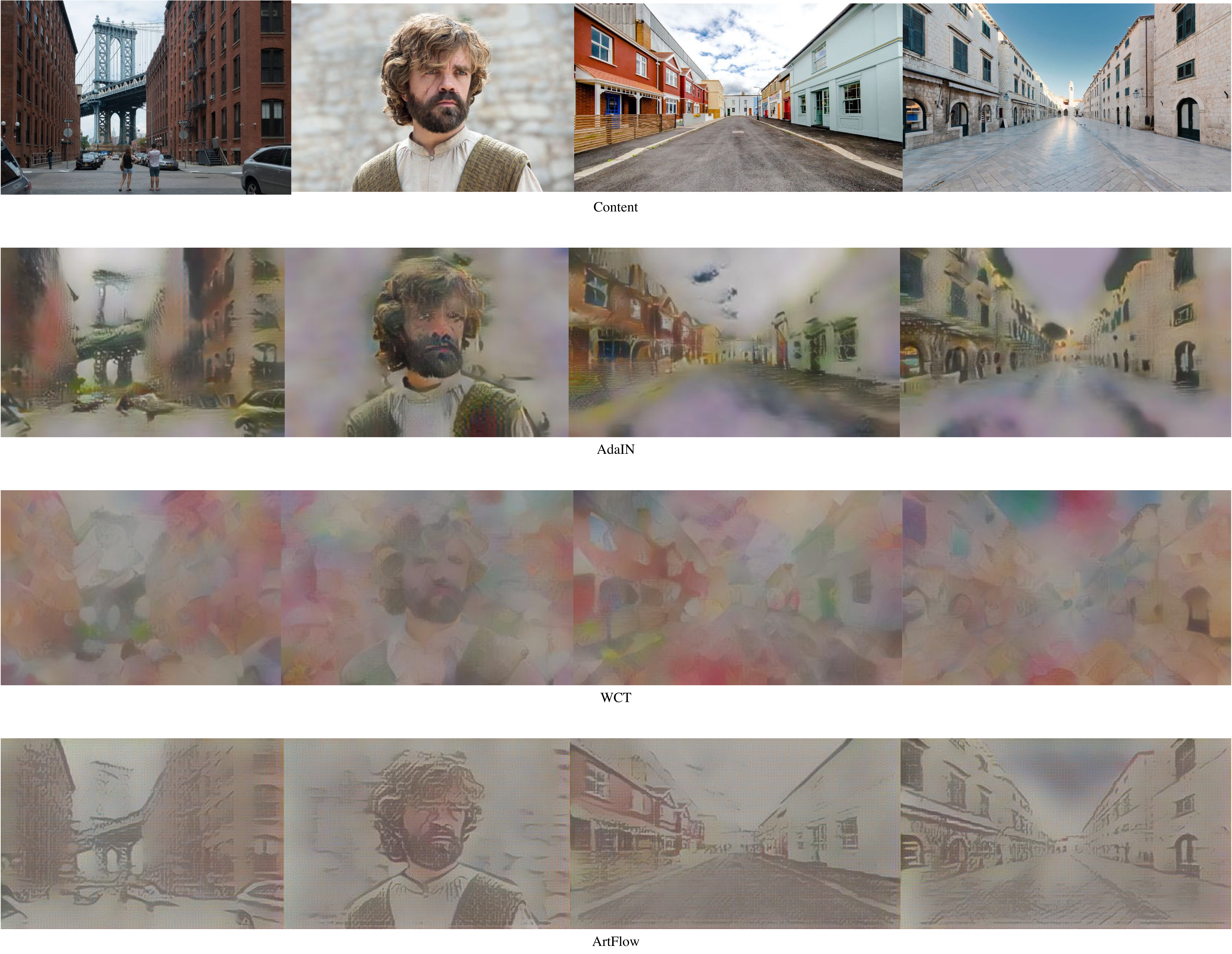}
	\caption{A comparison of the reconstructed content factor.}
%	\vspace{20mm}
	\label{fig:separation}
\end{figure*}

\end{document}